\documentclass{article}

\usepackage{authblk}
\usepackage[utf8]{inputenc} % allow utf-8 input
\usepackage[T1]{fontenc}    % use 8-bit T1 fonts
\usepackage{url}            % simple URL typesetting
\usepackage{booktabs}       % professional-quality tables
\usepackage{amsfonts}       % blackboard math symbols
\usepackage{nicefrac}       % compact symbols for 1/2, etc.
\usepackage{microtype} 
\usepackage{graphicx}
\usepackage{amsmath}
\usepackage{bm}
\usepackage{amssymb}
\usepackage{amsthm}
\usepackage{mathdots}
\usepackage{caption}
\usepackage{multirow}
\usepackage[inline]{enumitem}
\usepackage{latexsym}
\usepackage{amsopn}
\usepackage{mathdots}
\usepackage{enumitem}
\usepackage{subcaption}
\usepackage{graphicx}
\usepackage{epstopdf}
\newtheorem{theorem}{Theorem}

\newtheorem{lemma}{Lemma}

\newtheorem{corollary}{Corollary}

\newtheorem{claim}{Claim}

\newtheorem{proposition}{Proposition}

\newcommand{\dotprod}[2]{\langle #1,#2 \rangle}
\newcommand{\E}{\mathbb{E}}
\newcommand{\Eh}{\widehat{\E}}
\renewcommand{\P}{\mathbb{P}}

\newcommand{\Z}{\mathcal{Z}}
\newcommand{\D}{\mathcal{D}}
\newcommand{\F}{\mathcal{F}}
\newcommand{\G}{\mathcal{G}}
\newcommand{\N}{\mathcal{N}}

\newcommand{\Rademacher}{\mathfrak{R}}
\newcommand{\Rademacherh}{\widehat{\Rademacher}}

\def\aa{\mbox{\boldmath $a$}}

\newcommand{\BB}{\mbox{\boldmath $B$}}

\newcommand{\dd}{\mbox{\boldmath $d$}}

\newcommand{\ee}{\mbox{\boldmath $e$}}

\renewcommand{\P}{\mbox{$\mathbb P$}}

\newcommand{\one}{\mbox{\boldmath $1$}}

\newcommand{\R}{\mbox{\boldmath $R$}}

\newcommand{\si}{\mbox{\boldmath $\sigma$}}

\newcommand{\VV}{\mbox{\boldmath $V$}}
\newcommand{\vv}{\mbox{\boldmath $v$}}
\newcommand{\w}{\mbox{\boldmath $w$}}
\newcommand{\W}{\mbox{\boldmath $W$}}
\newcommand{\X}{\mathcal{X}}
\newcommand{\x}{\mbox{\boldmath $x$}}

\def\BB{\mbox{$\mathbf B$}}

\def\D{\mbox{$\mathcal D$}}
\def\F{\mbox{$\mathcal F$}}

\def\MM{\mbox{$\mathbf M$}}

\def\si{\mbox{$\sigma$}}
\def\uu{\mbox{$\mathbf u$}}
\def\VV{\mbox{$\mathbf V$}}

\def\xx{\mbox{$\mathbf x$}}
\def\u{\mbox{$\mathbf u$}}

\def\WW{\mbox{$\mathbf W$}}
\def\zz{\mbox{$\mathbf z$}}

\def\R{\mbox{$\mathbb R$}}

\def\Y{\mbox{$\mathcal Y$}}
\def\zer{\mbox{\boldmath $0$}}

\def\E{ \mathbb{E}}

\newcommand{\cm}[1]{}
\newcommand{\lm}[1]{}
\newcommand{\om}[1]{}

\newcommand{\norm}[1]{\left\Vert#1\right\Vert}

\begin{document}
\title{Understanding Weight Normalized Deep Neural Networks with Rectified Linear Units}
	\author[1]{Yixi Xu}
\author[1]{Xiao Wang}
\affil[ ]{xu573@purdue.edu, wangxiao@purdue.edu }
\affil[1]{Department of Statistics, Purdue University, West Lafayette, IN 47907, USA}
\date{}
\maketitle
\begin{abstract}
This paper presents a general framework for norm-based capacity control for $L_{p,q}$ weight normalized deep neural networks. We establish the upper bound on the Rademacher complexities of this family. With an $L_{p,q}$ normalization where $q\le p^*$ and $1/p+1/p^{*}=1$, we discuss properties of a width-independent capacity control, which only depends on the depth by a square root term. We further analyze the approximation properties of $L_{p,q}$ weight normalized deep neural networks. In particular, for an $L_{1,\infty}$ weight normalized network, the approximation error can be controlled by the $L_1$ norm of the output layer, and the corresponding generalization error only depends on the architecture by the square root of the depth.
%
%We argue that the approximation error can be controlled by the $L_1$ norm of the output layer for any large enough $L_{1,\infty}$ weight normalized neural network, while the corresponding generalization error only depends on the architecture by the square root of the depth.
\end{abstract}

\section{Introduction}
During the past decade, deep neural networks (DNNs) have demonstrated an amazing performance in 
solving many complex artificial intelligence tasks such as object recognition and   
identification, text understanding and translation, question answering, and more \cite{goodfellow2016deep}. The capacity of {\it unregularized} fully connected DNNs, as a function of the network size and depth, is fairly well understood \cite{anthony2009neural,bartlett1998sample,shalev2014understanding}. 
By bounding the $L_2$ norm of the incoming weights of each unit,  \cite{salimans2016weight} is able to accelerate the convergence of stochastic gradient descent optimization across applications in supervised image recognition, generative modeling, and deep reinforcement learning. However, theoretical investigations on such networks are less explored in the literature, and a few exceptions are \cite{bartlett1998sample,bartlett2017spectrally,pmlr-v75-golowich18a,neyshabur2018a,neyshabur2015norm,sun2016depth}. There is a central question waiting for an answer: Can we bound the capacity of fully connected DNNs with bias neurons by weight normalization alone, which has the least dependence on the architecture?

In this paper, we focus on networks with rectified linear units (ReLU) and study a more general weight normalized deep neural network (WN-DNN), 
which includes all layer-wise $L_{p,q}$ weight normalizations.  In addition, these networks have a bias neuron per hidden layer, while prior studies \cite{bartlett1998sample,bartlett2017spectrally,pmlr-v75-golowich18a,neyshabur2018a,neyshabur2015norm,sun2016depth} either exclude the bias neuron, or only include the bias neuron in the input layer, which differs from the practical application.  We establish the upper bound on the Rademacher complexities of this family and study the theoretical properties of WN-DNNs in terms of the approximation error.

 We first examine how the  $L_{p,q}$ WN-DNN architecture  influences their generalization properties. Specifically, for $L_{p,q}$ normalization where $q\le p^*$ and $1/p+1/p^{*}=1$, we obtain a complexity bound that is independent of width and only has a square root dependence on the depth. To the best of our knowledge, this is the first theoretical result for the fully connected DNNs including a bias neuron for each hidden layer in terms of generalization. We will demonstrate later that it is nontrivial to extend the existing results to the DNNs with bias neurons. Even excluding the bias neurons, existing generalization bounds for DNNs depend on either width or depth logarithmically \cite{bartlett2017spectrally}, polynomially\cite{pmlr-v75-golowich18a,neyshabur2018a}, or even exponentially \cite{neyshabur2015norm,sun2016depth}. Even for \cite{bartlett2017spectrally}, the logarithmic dependency is not always guaranteed, as  the margin bound is \[O\left(\log (\max{\dd})/\sqrt{n}\prod\limits_{i=1}^k \norm{\WW_i}_{\si} \left(\sum\limits_{i=1}^k \norm{\WW^T_i-\MM^T_i}^{2/3}_{2,1}/\norm{\WW_i}_{\si}^{2/3}\right)^{3/2} \right),\] where $\norm{\cdot}_{\si}$ is the spectral norm, and $\MM_i$ is a collection of predetermined reference matrix. The bound will worsen, when the $\WW_i$ moves farther from $\MM_i$. For example, if  \[\norm{\WW^T_i-\MM^T_i}_{2,1}/\norm{\WW_i}_{\si}\ge A_0\] for some constant $A_0$, then the above bound will rely on the network size by $O\left(\log (\max{\dd})k^{3/2} \right)$.
 
We also examine the approximation error of WN-DNNs. It is shown that the $L_{1,\infty}$ WN-DNN is able to approximate any Lipschitz continuous function arbitrarily well by increasing the norm of its output layer and growing its size. Early work on neural network approximation theory includes the universal approximation theorem  \cite{cybenko1989approximation,hornik1991approximation,pinkus1999approximation}, indicating that a fully connected network with a single hidden layer can approximate any continuous functions.
More recent work expands the result of shallow networks to deep networks with an increased interest in the expressive power of deep networks  especially for some families of "hard" functions \cite{arora2018understanding,Eldan2016ThePO, liang2016deep,pmlr-v70-safran17a,pmlr-v49-telgarsky16,yarotsky2017error}. For instance, \cite{pmlr-v49-telgarsky16} shows that for any positive integer $l$, there exist neural networks with $\Theta(l^3)$ layers and $\Theta(1)$ nodes per layer, which can not be approximated by networks with $\Theta(l)$ layers unless they possess $\Omega(2^l)$ nodes. These results on the other hand request for an artificial neural network of which the generalization bounds grow slowly with depth and even avoid explicit dependence on depth. 

The contributions of this paper are summarized as follows. \begin{enumerate}
\item  We extend the $L_{2,\infty}$ weight normalization \cite{salimans2016weight} to the more general $L_{p,q}$ WN-DNNs and relate these classes to those represented by unregularized DNNs. \item We include a bias node not only in the input layer but also in every hidden layer. As discussed in Claim \ref{claim:1}, it is nontrivial to extend prior research to study this case.
\item  We study the Rademacher complexities of WN-DNNs. Especially,  with any $L_{p,q}$ normalization satisfying that $q\le p^*$, we have a capacity control that is independent of the width and depends on the depth by $O(\sqrt{k})$.
\item  We analyze the approximation property of $L_{p,q}$ WN-DNNs and further show the theoretical advantage of  $L_{1,\infty}$ WN-DNNs. 
\end{enumerate}
The paper is organized as follows.  In Section 2, we define the $L_{p,q}$ WN-DNNs and analyze the corresponding function class. Section 3 gives the Rademacher complexities. In Section 4, we provide the error bounds for the approximation error of Lipschitz continuous functions. 
\section{Preliminaries}\label{sec:prelim}
In this section, we define the WN-DNNs, of which the weights and biases for all layers are scaled by some norm up to a normalization constant $c$. Furthermore, we demonstrate how it surpasses unregularized DNNs theoretically. 

A neural network on $\R^{d_0}\to\R^{d_{k+1}}$ with $k$ hidden layers is defined by a set of $k+1$ affine transformations $T_1: \R^{d_0}\to\R^{d_1},  T_2: \R^{d_1}\to\R^{d_2},\cdots,  T_{k+1}: \R^{d_k}\to\R^{d_{k+1}}$ and the ReLU activation $\sigma(u)=(u)_{+}=u I\{u>0\}$. The affine transformations are parameterized by $T_i(\uu)=\W_i^T\uu+\BB_i$, where $\W_{i}\in \R^{d_{i-1}\times d_i},\BB_i\in \R^{d_i}$ for $i=1,\cdots,k+1$.
	The function represented by this neural network is
	\[
	f(x)=T_{k+1}\circ\si\circ T_{k}\circ\cdots\circ\si\circ T_{1}\circ\x
	\]

Before introducing $L_{p,q}$ WN-DNNs, we build an augmented layer for each hidden layer by appending the bias neuron $1$ to the original layer, then combine the weight matrix and the bias vector as a new matrix. 

Define $f^{*}_0(\x)=(1,\x^T)^T$. Then the first hidden layer \[f_1(\x)= 
	T_{1}\circ\x\triangleq \tilde{\VV}_1^Tf_0^*(\x),\] where  $\tilde{\VV}_1=(\BB_1,\W_1^T)^T\in \R^{(d_0+1)\times d_1}$. Define the augmented first hidden layer as
	\[f^{*}_1(\x)=(1,(f_1(\x))^T)^T\in \R^{d_1+1}.\] Then $f^{*}_1(\x)\triangleq\VV_1^Tf_0^{*}(\x)$, where $\VV_1=(\ee_{10},\tilde{\VV}_1)\in \R^{(d_0+1)\times (d_1+1)}$ and $\ee_{10}=(1,0,\cdots,0)^T\in \R^{d_0+1}$. 
    Sequentially for $i=2,\cdots,k$, define the $i$th hidden layer as 
	\begin{equation}\label{eqn:def:f}
	f_i(\x)= T_{i}\circ \si\circ  f_{i-1}(\x)\triangleq \dotprod{\tilde{\VV}_i}{\si\circ f^{*}_{i-1}(\x)}, 
	\end{equation}
	where  $
	\tilde{\VV}_i=(\BB_i,\W_i^T)^T\in \R^{(d_{i-1}+1)\times d_i}$. Note that $\si(1)=1$, thus $(1,\si\circ f_{i-1}(\x))=\si\circ f^{*}_{i-1}(\x)$.
	The augmented $i$th hidden layer is
	\begin{equation}\label{eqn:aughiddenlayer}
	f_i^{*}(\x)=(1,(f_i(\x))^T)^T\in \R^{d_i+1},
	\end{equation} and $f^{*}_i(\x)\triangleq  \dotprod{\VV_i}{\si\circ f^{*}_{i-1}(\x)}$, where 
\begin{equation}\label{eqn:def:V}
	\VV_i=(\ee_{1i},\tilde{\VV}_i)\in \R^{(d_{i-1}+1)\times (d_i+1)},
	\end{equation} and $\ee_{1i}=(1,0,\cdots,0)^T\in \R^{d_{i-1}+1}$. The output layer is  
	\begin{equation}
	f(\x)=T_{k+1} \circ \si \circ f^{*}_{k}(\x)\triangleq \dotprod{\tilde{\VV}_{k+1}}{\si\circ f^{*}_{k}(\x)},
	 \end{equation}
	 where  $
	 \tilde{\VV}_{k+1}=(\BB_{k+1},\W_{k+1}^T)^T\in \R^{(d_{k}+1)\times d_{k+1}}$.
\paragraph{The $L{p,q}$ Norm.}
	The $L{p,q}$ norm of a $s_1\times s_2$ matrix $A$ is defined as 
	\[\norm{A}_{p,q}=\left( \sum\limits_{j=1}^{s_2}  \left( \sum\limits_{i=1}^{s_1}|a_{ij}|^p\right)^{q/p}\right)^{1/q},\] where $1\le p <\infty$ and $1\le q \le \infty$.  When $q=\infty$, $\norm{A}_{p,\infty}= \sup_{j} \left( \sum\limits_{i=1}^{s_1}|a_{ij}|^p\right)^{1/p}$. When $p=q=2$, the $L_{p,q}$ is the Frobenius norm. 
	
We motivate our introduction of WN-DNNs with a negative result when directly applying existing studies on fully connected DNNs with bias neurons. 

\paragraph{A Motivating Example.}
As shown in Figure \ref{fig:1}, define $f=T_2\circ\si\circ T_1:\R\to\R$, where $T_1(x)=(-x+1,-x-1)\triangleq \tilde{\VV}_1^T(1,x)^T$ and $T_2(\uu)=1-u_1-u_2\triangleq\tilde{\VV}_2^T(1,u_1,u_2)^T$. 
Consider $f^{'}=100T_2\circ\si\circ \frac{1}{100}T_1$, as shown in Figure \ref{fig:2} . Then 
\[f^{'}(x)=100-\si(-x+1)-\si(-x-1)=99+f(x)\]
Note that the product of the norms of all layers for $f^{'}$ remains the same as that for $f$: \[\norm{100T_2}_**\norm{\frac{T_1}{100}}_*=\norm{T_2}_**\norm{T_1}_*,\]
where the norm of the affine transformation $\norm{T_i}_*$ is defined as the norm of its corresponding linear transformation matrix $\norm{\tilde{\VV}_i}_*$ for $i=1,2$. Using a similar trick, we could replace the 100 in this example with any positive number. This on the other hand suggests an unbounded output even when the product of the norms of all layers is small. 

\begin{figure}[h]
\centering
\begin{subfigure}{0.45\textwidth}
\centering
	\includegraphics[scale=0.7]{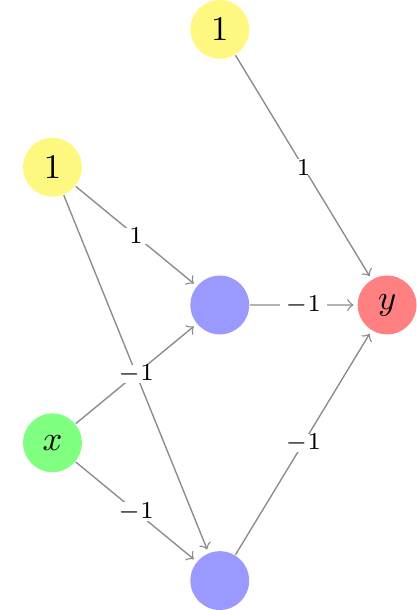}
	\caption{Visualization of $f$.}
	\label{fig:1}
	\end{subfigure}
	\begin{subfigure}{0.45\textwidth}
	\centering
	\includegraphics[scale=0.7]{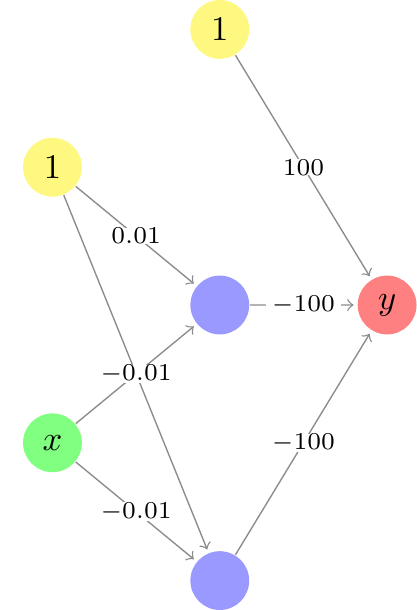}
	\caption{Visualization of $f^{'}$.}
	\label{fig:2}
	\end{subfigure}
	\caption{The motivating example.}
\end{figure}

Furthermore, a negative result will be presented in terms of Rademacher complexity in the following claim. 
\begin{claim}\label{claim:1}
	Define $\N_{\gamma_{*}\le \gamma}^{k,\dd}$ as a function class that contains all functions representable by some neural network of depth $k+1$ and widths $\dd$: $f=T_{k+1}\circ\si\circ T_{k}\circ\cdots\circ\si\circ T_{1}\circ\x$, where $\dd=(m_1,d_1,\cdots,d_k,1),$ $\norm{\cdot}_*$ is an arbitrary norm, and $T_i(\uu): \R^{d_{i-1}}\to\R^{d_{i}}=\tilde{\VV}_i^T(1,\uu^T)^T$, for $i=1,\cdots,k+1$, such that 
	\[\gamma_*=\prod\limits_{i=1}^{k+1}\norm{\tilde{\VV}_i}_{*}\le \gamma.\]
	Then for a fixed n and any sample  $S=\{\x_1,\cdots,\x_n\}\subseteq \R^{m_1}$,
	\[\Rademacherh_S(\N_{\gamma_{*}\le \gamma}^{k,\dd})=\infty.\]
\end{claim}

Claim \ref{claim:1} shows the failure of current norm-based constraints on fully connected neural networks with the bias neuron in each hidden layer.   Prior studies \cite{bartlett1998sample,bartlett2017spectrally,pmlr-v75-golowich18a,neyshabur2018a,neyshabur2015norm,sun2016depth} included the bias neuron only in the input layer and considered layered networks parameterized by a sequence of weight matrices only, that is $\BB_i=\zer$ for all $i=1, \cdots, k+1$. While fixing the architecture of neural networks, these works imply that $\prod\limits_{i=1}^{k+1}\norm{\WW_i}_*$  is sufficient to control the Rademacher complexity of the function class represented by these DNNs, where $\norm{\cdot}_{*}$ is the spectral norm in \cite{bartlett2017spectrally,neyshabur2018a}, the $L_{1,\infty}$ norm in \cite{bartlett1998sample,sun2016depth}, the $L_{1,\infty}/L_{2,2}$ norm in \cite{pmlr-v75-golowich18a}, and 
the $L_{p,q}$ norm in \cite{neyshabur2015norm} for any $p\in[1,\infty)$, $q\in[1,\infty]$. However, this kind of control fails once the  bias neuron is added to each hidden layer, demonstrating the necessity to use WN-DNNs instead.
\paragraph{The $L_{p,q}$ WN-DNNs. }
	An $L_{p,q}$ WN-DNN by a normalization constant $c\ge 1$ with $k$ hidden layers is defined by a set of $k+1$ affine transformations $T_1: \R^{d_0}\to\R^{d_1},   T_2: \R^{d_1}\to\R^{d_2},\cdots,  T_{k+1}: \R^{d_k}\to\R^{d_{k+1}}$ and the ReLU activation, where $T_i(\uu)=\tilde{\VV}_i^T(1,\uu^T)^T$, $\tilde{\VV}_i\in\R^{(d_{i-1}+1)\times d_i}$ and $\norm{T_i}_{p,q}\triangleq\norm{\tilde{\VV}_i}_{p,q}$, for $i=1,\cdots,k+1$. In addition, $\norm{T_i}_{p,q}\equiv c$ for $i=1,\cdots,k$.

	Define $\N_{p,q,c,c_o}^{k,\dd}$ as the collection of all functions that could be represented by an $L_{p,q}$ WN-DNN with the normalization constant $c$ satisfying: 
	\begin{enumerate}%[itemjoin=\qquad]
		\item[(a)] The number of neurons in the $i$th hidden layer is $d_i$ for $i=1,2,\cdots, k$. The dimension of input is $d_0$, and output $d_{k+1}$;
			\item[(b)] It has $k$ hidden layers;
		\item[(c)] $\norm{T_i}_{p,q}\equiv c$  for $i=1,\cdots,k$; \item[(d)] $\norm{T_{k+1}}_{p,q}\le c_o$.
		\end{enumerate}
    The following theorem provides some useful observations regarding $\N_{p,q,c,c_o}^{k,\dd}$.
\begin{theorem}\label{thm:nn:prop}
	Let $c,c_o,c_1,c_2,c_o^1,c_o^2>0$, $p\in[1,\infty)$, $q\in[1,\infty]$, $k,k_1,k_2\in \mathbb{N}$, $\dd=(d_0, d_1\cdots,d_{k+1})\in\mathbb{N}_+^{k+2}$, $\dd^1=(d^1_0, d^1_1\cdots,d^1_{k_1+1})\in\mathbb{N}_+^{k_1+2}$, and $\dd^2=(d^2_0, d^2_1\cdots,d^2_{k_2+1})\in\mathbb{N}_+^{k_2+2}$.
	\begin{enumerate}[label=(\alph*)]
		\item \label{nn:prop:1}
		A function $f:\R^{d_0}\to\R^{d_{k+1}}=T_{k+1}\circ\si\circ T_{k}\circ\cdots\circ\si\circ T_{1}\circ\x$, where $T_i(\uu)=\W_i^T\uu+\BB_i:\R^{d_{i-1}}\to\R^{d_i}$. Then $f\in \N_{p,q,c,c_o}^{k,\dd}$, as long as $\norm{T_i}_{p,q}\le c$ for $i=1,\cdots,k$ and $\norm{T_{k+1}}_{p,q}\le c_o$.
		\item \label{nn:prop:2}$ \N_{p,q,c_1,c_o}^{k,\dd}\subseteq \N_{p,q,c_2,c_o}^{k,\dd}$ if $c_1\le c_2$. 
		$ \N_{p,q,c,c^1_o}^{k,\dd}\subseteq \N_{p,q,c,c^2_o}^{k,\dd}$ if $c^1_o\le c^2_o$. 
		If $g\in \N_{p,q,c,1}^{k,\dd}$, then $c_og\in\N_{p,q,c,c_o}^{k,\dd}$.
		\item   \label{nn:prop:norm} $ \N_{p_1,q,c,c_o}^{k,\dd}\subseteq \N_{p_2,q,c,c_o}^{k,\dd}$ if $1\le p_1\le p_2<\infty$. $ \N_{p,q_1,c,c_o}^{k,\dd}\subseteq \N_{p,q_2,c,c_o}^{k,\dd}$ if $1\le q_1\le q_2\le \infty$.
		
		$ \N_{p,\infty,c,c_o}^{k,\dd}\subseteq \N_{p,q,\tilde{c},\tilde{c}_o}^{k,\dd}$, where  $\tilde{c}=c\max^{\frac{1}{q}}\{d_1,d_2\cdots,d_{k}\}$ and $\tilde{c}_o=d_{k+1}^{\frac{1}{q}}c_o$. Especially, when $d_{k+1}=1$, $\tilde{c}_o=c_o$.
		\item \label{nn:prop:widendeepen}$ \N_{p,q,c,c_o}^{k_1,\dd^1}\subseteq \N_{p,q,c,c_o}^{k_2,\dd^2}$ if $c\ge1$, $k_1\le k_2$, $d_0^2=d_0^1$, $d_i^2\ge d_i^1$ for $i=1,\cdots,k_1$, $ d_i^2\ge d_{k_1+1}^1$ for  $i>k_1$, and $d_{k_2+1}^2=d_{k_1+1}^1=1$.
	\end{enumerate}
\end{theorem}

 In particular, Part (a) connects normalized neural networks to unregularized DNNs. Part (b) shows the increased expressive power of neural networks by increasing the normalization constant or the output layer norm constraint. Part (c) discusses the influence of the choice of $L_{p,q}$ normalization on its representation capacity.
Part (d) describes the gain in representation power by either widening or deepening the neural networks. 

\section{Estimating the Rademacher Complexities of $\N_{p,q,c,c_o}^{k,\dd}$}\label{sec:rader}
In this section, we bound the Rademacher complexities of $\N_{p,q,c,c_o}^{k,\dd}$, where $d_0=m_1$ and $d_{k+1}=1$. Without loss of generality, assume the input space $\mathcal{X}= [-1,1]^{m_1}$ in the following sections. Further define $p^*$ by $1/p+1/p^{*}=1$.
\begin{proposition}\label{prop:RCp1}
	Fix $q \ge 1, k \ge 0, c, c_o>0, d_i\in\mathbb{N}_+$ for $i=1,\cdots,k$, then for any set $S=\{\x_1,\cdots,\x_n\}\subseteq \X$ , we have
\begin{align*}
\Rademacherh_S(\N_{1,q,c,c_o}^{k,\dd})&\le\frac{c_o}{\sqrt{n}}\min\left( 2\max(1,c^k)\sqrt{k+2+\log(m_1+1)},\right.\\&\left.\sqrt{k\log 16}\sum\limits_{i=0}^kc^i+c^k(\sqrt{2\log(2m_1)}+\sqrt{k\log 16})\right).\end{align*}
\end{proposition}
\begin{proof}[Proof sketch]
	As $\si(1)=1$, we could treat the bias neuron in the $i$th hidden layer as a hidden neuron computed from the $(i-1)$th hidden layer by \[\si(\ee_{1i}^Tf^*_{i-1}(\x))=1,\] where $\ee_{1i}=(1,0,\cdots,0)^T\in \R^{d_{i-1}+1}$, and $f^*_{i-1}$ is the augmented $(i-1)$th hidden layer as defined in Equation \eqref{eqn:aughiddenlayer}. 
	Therefore, the new affine transformation could be parameterized by $\VV_i$ defined in Equation \eqref{eqn:def:V}, such that $\norm{\VV_i}_{1,\infty}=\max(1,c)$. Then the result is the minimum of the bound of \cite[Theorem 2]{pmlr-v75-golowich18a} on DNNs without bias neurons and that of Proposition \ref{prop:RCp2} when $p=1$.
	\end{proof}
\begin{proposition}\label{prop:RCp2}
	Fix $p, q \ge 1, k \ge 0, c,  c_o>0, d_i\in\mathbb{N}_+$ for $i=1,\cdots,k$, then for any set $S=\{\x_1,\cdots,\x_n\}\subseteq \X$ , we have 
	\begin{enumerate}[label=(\alph*)]
	\item for $p\in (1,2]$, 
	\begin{equation}\label{eqn:prop2:1}
		\begin{split}
	\Rademacherh_S(\N_{p,q,c,c_o}^{k,\dd})&\le c_o\sqrt{\frac{(k+1)\log 16}{n}}\left(\sum\limits_{i=1}^{k+1}c^{k-i+1}\prod\limits_{l=i}^kd_l^{[\frac{1}{p^*}-\frac{1}{q}]_+}\right)+\\&\frac{c_oc^{k}}{\sqrt{n}}\prod\limits_{i=1}^kd_i^{[\frac{1}{p^*}-\frac{1}{q}]_+}m_1^{\frac{1}{p^{*}}}\left[\min\left((\sqrt{p^{*}-1}, \sqrt{2\log(2m_1)} \right)+\sqrt{(k+1)\log 16}\right],
	\end{split}
\end{equation}
\item   for $p\in 1\cup(2,\infty)$, 
		\begin{equation}\label{eqn:prop2:2}
	\begin{split}
\Rademacherh_S(\N_{p,q,c,c_o}^{k,\dd})&\le c_o\sqrt{\frac{(k+1)\log 16}{n}}\left(\sum\limits_{i=1}^{k+1}c^{k-i+1}\prod\limits_{l=i}^kd_l^{[\frac{1}{p^*}-\frac{1}{q}]_+}\right)+\\&\frac{1}{\sqrt{n}}c_oc^{k}\prod\limits_{i=1}^kd_i^{[\frac{1}{p^*}-\frac{1}{q}]_+}m_1^{\frac{1}{p^{*}}}\left(\sqrt{2\log(2m_1)}+\sqrt{(k+1)\log 16}\right).
\end{split}
	\end{equation}
		\end{enumerate}
\end{proposition}
\begin{proof}[Proof sketch]
	The proof consists of two steps. 
	In the first step, following the notations in Section 2, we define a series of random variables \[Z_{j}=\sup_{f\in\N^{k,\dd}_{p,q,c,c_o}}\norm{\sum\limits_{i=1}^n\epsilon_i\si\circ f_j(\x_i)}_{p^*},\] where $\{\epsilon_1,\cdots,\epsilon_n\}$ are $n$ i.i.d Rademacher random variables, and $f_j$ is the $j$th hidden layer of the neural network $f$. We prove by induction that  for any $t\in\R$, \[\E_\epsilon\exp(tZ_j)\le4^{j} \exp\left(\frac{t^2ns^2_j}{2}+tc^j\prod\limits_{i=1}^jd_i^{[\frac{1}{p^*}-\frac{1}{q}]_+}A_{m_1,S}^{p}\right),\] where \[s_j=\sum\limits_{i=2}^jc^{j-i+1}\prod\limits_{l=i}^jd_l^{[\frac{1}{p^*}-\frac{1}{q}]_+}+(m_1^{1/p^*}+1)c^j\prod\limits_{l=1}^jd_l^{[\frac{1}{p^*}-\frac{1}{q}]_+},\] and $A_{m_1,S}^{p}$ is some constant only depends on the sample. In addition, we relies on H\"older's inequality with an optimal parameter to separate the bias neuron. 
	Step 2 is motivated by the idea of  \cite{pmlr-v75-golowich18a}.  By Jensen's inequality \[	n\Rademacherh_S(\N_{p,q,c,c_o}^{k,\dd})\le \frac{1}{\lambda}\log\E_\epsilon \exp\left(\lambda\sup_{f \in \N_{p,q,c,c_o}^{k,\dd}}{\left(  \sum_{i=1}^n{\epsilon_i f(\x_i)}\right)}\right).\] 
	Finally we get the desired result by choosing the optimal $\lambda$.
\end{proof}
When $\dd=d \one$, the upper bound of Rademacher complexity depends on the width by $O(d^{k[ \frac{1}{p^*}-\frac{1}{q} ]_+})$, which is similar to the case without bias neurons \cite{neyshabur2015norm}. Furthermore, the dependence on widths 
disappears as long as $q\in[1,p^*]$. 
 In order to investigate the tightness of the bound given in Proposition \ref{prop:RCp2}, we consider the binary classification as a specific case, indicating that when $\frac{1}{p} + \frac{1}{q} < 1$, the dependence on width is unavoidable. 
\begin{proposition}\cite[Theorem 3]{neyshabur2015norm} \label{prop:lower}
	For any $p,q\ge 1$, $\dd=d \one$ and any $k\ge 2$, $n$ $\{-1,+1\}$ points could be shattered with unit margin by $\N_{p,q,c,c_o}^{k,\dd}$, with 
	\[c^kc_o\le (\log_2n)^{\frac{1}{p
	}} n^{(\frac{1}{p}+\frac{1}{q})}d^{-(k-2)[ \frac{1}{p^*}-\frac{1}{q} ]_+}.\]
\end{proposition}
\paragraph{Issues on Bias Neurons.}
$L_{p,q}$ norm-constrained fully connected DNNs with no bias neuron were investigated in prior studies \cite{bartlett1998sample,pmlr-v75-golowich18a,neyshabur2015norm,sun2016depth}. First of all, the generalization bounds given by \cite{bartlett1998sample,neyshabur2015norm,sun2016depth} have explicit exponential dependence on the depth, thus it is not meaningful to compare these results with ours. Secondly, \cite{pmlr-v75-golowich18a} provides the up-to-date Rademacher complexity bounds of both $L_{1,\infty}$ and $L_{2,2}$ norm-constrained fully connected DNNs without bias neurons. However, it is not straightforward to extend their results to fully connected DNNs with a bias neuron in each hidden layer. For example, consider the $L_{2,2}$ WN-DNNs with $c=1$. If we simply treat each bias neuron as a hidden neuron, as in the proof for Proposition \ref{prop:RCp1}, the complexity bounds \cite{pmlr-v75-golowich18a} grows exponentially with respect to the depth by $O(\sqrt{k}2^{\frac{k}{2}})$, while our Proposition \ref{prop:RCp2} gives a much tighter bound $O(k^{\frac{3}{2}})$.
\paragraph{Comparison with \cite{pmlr-v75-golowich18a} on the Rademacher compexity bounds of $L_{1,\infty}$ and $L_{2,2}$ WN-DNNs.}
\cite{pmlr-v75-golowich18a} is the most recent work on the Rademacher complexities of the $L_{1,\infty}$ and $L_{2,2}$ norm-constrained fully connected DNNs without bias neurons.  Consider a specific case when $\log(m_1)$ is small and $c_o=1$ to shed light on the possible influence of the bias neurons on the generalization properties.

\begin{table}[h]
\centering
		\begin{tabular}{c|cc}
			\hline
			&With Bias Neurons & Without Bias Neurons \cite{pmlr-v75-golowich18a}  \\ \hline
			$c<1$&$O(\frac{\sqrt{k}(1-c^{k+1})}{(1-c)\sqrt{n}})$ &$O(\frac{\sqrt{k}c^k}{\sqrt{n}})$ \\\hline
			$c=1,L_{1,\infty}$&$O(\frac{\sqrt{k}}{\sqrt{n}})$&$O(\frac{\sqrt{k}}{\sqrt{n}})$\\\hline
			$c=1,L_{2,2}$&$O(\frac{k^{3/2}}{\sqrt{n}})$&$O(\frac{\sqrt{k}}{\sqrt{n}})$\\\hline
			$c>1, L_{1,\infty}$&$O(\frac{\sqrt{k}c^{k}}{\sqrt{n}})$&$O(\frac{\sqrt{k}c^k}{\sqrt{n}})$
			\\ \hline
				$c>1,L_{2,2}$&$O(\frac{\sqrt{k}(c^{k+1}-1)}{\sqrt{n}})$&$O(\frac{\sqrt{k}c^k}{\sqrt{n}})$
			\\ \hline
		\end{tabular}
		\caption{Rademacher complexity bounds for $L_{1,\infty}$/$L_{2,2}$ WN-DNNs with/without bias neurons. }
		\label{tab:1}
\end{table}

As summarized in Table \ref{tab:1}, these comparisons suggest that the inclusion of a bias neuron in each hidden layer might lead to extra dependence of generalization bounds on the depth especially when $c$ is small. Note that, when $c<1$, $\sqrt{k}(1-c^{k+1})/(1-c)\to \infty$, while $\sqrt{k}c^k\to 0$, as $k\to \infty$. For $L_{2,2}$ WN-DNNs, when $c=1$, the bounds are $O(\frac{k^{\frac{3}{2}}}{\sqrt{n}})$ if with bias neurons and $O(\frac{\sqrt{k}}{\sqrt{n}})$ without bias neurons. For $L_{2,2}$ WN-DNNs, when $c>1$, the bounds are $O(\frac{\sqrt{k}(c^{k+1}-1)}{\sqrt{n}})$ if including bias neurons and $O(\frac{\sqrt{k}c^k}{\sqrt{n}})$ if excluding bias neurons. Another interesting observation is that the complexity bounds remain the same no matter whether bias neurons are included or not, when $c>1$ for $L_{1,\infty}$ WNN-DNNs. 

\section{Approximation Properties }\label{sec:approx}
In this section, we analyze the approximation properties of $L_{p,q}$ WN-DNNs and show the theoretical advantage of $L_{1,\infty}$ WN-DNN. We first introduce a technical lemma, demonstrating that any wide one-hidden-layer neural network could be exactly represented by a deep but narrow normalized neural network. In addition, Lemma \ref{lemma:approx:con:deep} indicates that  $\N_{1,\infty,\cdot,c_o}^{1,(m_1,r,1)}\subseteq \N_{p,\infty,1,2c_o}^{k,(m_1,([r/k]+2m_1+3){\bf1}_k,1)}$ for any $r>1$, $k\in \N$, and $c_o>0$, where $[x]$ is the smallest integer which is greater than or equal to $x$, and ${\bf1}_k=(1,\cdots,1)\in\R^k$.
\begin{lemma}\label{lemma:approx:con:deep}
	Assume that a function \[g(\x):\R^{m_1}\to\R=\sum\limits_{i=1}^rc_i\si (\w_i^T\x+b_i)\] satisfies that $\sum\limits_{i=1}^r |c_i|\le c_o$ and  $\norm{(b_i,\w_i^T)}_1=1$. Then for any integer $k\in[1,r]$,  \[g\in\N_{p,q,wid_k^{1/q},2c_o}^{k,\dd^k},\] where $wid_k=[r/k]+2m_1+3$, 
	$d^k_0=m_1$, $d^k_i=wid_k$ for $i=1,\cdots,k$, and $d^k_{k+1}=1$.
\end{lemma}
\begin{proof}[Proof sketch]
	Note that the shallow neural network $g$ could be decomposed as
	\[\sum\limits_{i=1}^{r_1}c_i^+\si\left((\w_i^+)^T\x+b_i^+\right)-\sum\limits_{i=1}^{r_2}c_i^-\si\left((\w_i^-)^T\x+b_i^-\right),
	\]
	where $c_i^+, c_i^->0$ and $r_1+r_2=r$. We consider a simplified case when  $g(\x)=\sum\limits_{i=1}^{r_1}c_i^+\si\left((\w_i^+)^T\x+b_i^+\right)$ to illustrate the main idea of our proof. Without loss of generality, assume that $\norm{(b_i,2\w_{i}^T)}_1=1$. First create a set  \[\mathcal{C}=\{\si\left((\w_i^+)^T\x+b_i^+\right),i=1,\cdots,r_1\}.\] In order to build a $k+1$-layer WN-DNN to represent $g$, we partition $\mathcal{C}$ into $k$ equally sized subsets:  $\mathcal{C}_1,\cdots,\mathcal{C}_k$.  The key idea is to get all elements of $\mathcal{C}_j$ in the $j$th hidden layer for $j=1,\cdots,k$, while keeping both $\si\circ\x$, and $\si\circ-\x$. In addition, the normalized cumulative sum $S_j$  of $\cup_{i\le j}\mathcal{C}_i$ is computed in the $j+1$th hidden layer.  More specifically, \[S_j=\frac{\sum\limits_{i=1}^{jr_1/k}c_i^+\si\left((\w_i^+)^T\x+b_i^+\right)}{\sum\limits_{i=1}^{jr_1/k}c_i^+}.\]
	Note that \[(\w_i^+)^T\x+b_i^+=(\w_i^+)^T\si\circ\x-(\w_i^+)^T\si\circ(-\x)+b_i^+,\] and  \[S_j=\frac{\sum\limits_{i=1}^{(j-1)r_1/k}c_i^+}{\sum\limits_{i=1}^{jr_1/k}c_i^+}\si(S_{j-1})+\sum\limits_{i=(j-1)r_1/k+1}^{jr_1/k}\frac{c_i^+}{{\sum\limits_{i=1}^{jr_1/k}c_i^+}}\si\left((\w_i^+)^T\x+b_i^+\right).\] Thus the $L_{1,\infty}$ norm of the corresponding transformation still $\le 1$. 
	\end{proof}
Based on Lemma \ref{lemma:approx:con:deep}, we establish that a WN-DNN is able to approximate any Lipschitz-continuous function arbitrarily well by loosing the constraint for the norm of the output layer and either widening or deepening the neural network at the same time. Especially, for $L_{p,\infty}$ WN-DNNs, the approximation error could be purely controlled by the norm of the output layer, while the $L_{p,\infty}$ norm of each hidden layer is fixed to be 1. 
\begin{theorem}\label{thm:approx}
	$f:\X\to\R$, satisfying that $\norm{f}_\infty\le L$, and $\lvert f(x)-f(y)\rvert\le L\norm{x-y}_\infty$. Then for any $p\in[1,\infty)$, $q\in[1,\infty]$, and any integer $k\in[1,C_r(m_1)(\log \frac{c_o}{L})^{-2(m_1+1)/(m_1+4)}\left(\frac{c_o}{L}\right)^{2(m_1+3)/(m_1+4)}] $, if $c_o$ greater than a constant depending only on $m_1$, there exists a function $h \in\N_{p,q,wid_k^{1/q},2c_o}^{k,\dd^k}$, where \[wid_k=[
	k^{-1}C_r(m_1)(\log \frac{c_o}{L})^{-\frac{2(m_1+1)}{m_1+4}}\left(\frac{c_o}{L}\right)^{\frac{2(m_1+3)}{m_1+4}}]+2m_1+3,\] 
	$\dd^k=(m_1,wid_k,\cdots,wid_k,1)$, such that 
	\[\sup\limits_{\norm{\x}_\infty\le 1}\vert f(\x)-h(\x)\rvert\le C(m_1)L(\frac{c_o}{L})^{-\frac{2}{m_1+1}}\log\frac{c_o}{L},	 \]
	where $C_r(m_1)$ and $C(m_1)$ denotes some constant that depends only on $m_1$.
\end{theorem}

Theorem \ref{thm:approx} shows that the approximation bounds could be controlled by $c_o$ given a sufficiently deep or wide  $L_{p,q}$ WN-DNN.  Assume that the loss function is 1-Lipschitz continuous, then the dependence of the corresponding generalization bound on the architecture for each $\N_{p,q,wid_k^{1/q},2c_o}^{k,\dd^k}$ defined above are summarized as follows:

\begin{enumerate}
	\item[(a)] $p=1, q=\infty$: $O\left(\sqrt{k}c_o\right)$;\\
	\item[(b)] $p=1, q<\infty$: $O\left(\sqrt{k}c_owid_k^{\frac{k}{q}}\right)$;\\
	\item[(c)] $p>1, q\in(p^*,\infty]$: $O\left(\sqrt{k}c_o[(1+wid_k)^{\frac{1}{p^*}}]^k\right)$;\\
	\item[(d)] $p>1, q\in[1,p^*]$: $O\left(\sqrt{k}c_o[(1+wid_k)^{\frac{1}{q}}]^k\right)$.
\end{enumerate}

\section{Concluding Remarks}
We present a general framework for capacity control on WN-DNNs. In particular, we provide a satisfying answer for the central question: we obtain the generalization bounds for $L_{1,\infty}$ WN-DNNs that grows with depth by a square root term while getting the approximation error controlled. It will be interesting to extend this work to mullticlass classification. However, if handling via Radermacher complexity analysis, the generalization bound will depend on the square root of the number of classes \cite{zhang2004statistical}. Besides the extension to convolutional neural networks, we are also working on the design of effective algorithms for $L_{1,\infty}$ WN-DNNs.
\subsubsection*{Acknowledgments}
We thank the anonymous reviewers for their careful reading of our manuscript and their insightful comments that have greatly improved the paper.

\bibliography{references} 

\begin{thebibliography}{10}

\bibitem{anthony2009neural}
Martin Anthony and Peter~L Bartlett.
\newblock {\em Neural network learning: Theoretical foundations}.
\newblock cambridge university press, 2009.

\bibitem{arora2018understanding}
Raman Arora, Amitabh Basu, Poorya Mianjy, and Anirbit Mukherjee.
\newblock Understanding deep neural networks with rectified linear units.
\newblock In {\em International Conference on Learning Representations}, 2018.

\bibitem{bach2017breaking}
Francis Bach.
\newblock Breaking the curse of dimensionality with convex neural networks.
\newblock {\em Journal of Machine Learning Research}, 18(19):1--53, 2017.

\bibitem{bartlett1998sample}
Peter~L Bartlett.
\newblock The sample complexity of pattern classification with neural networks:
  the size of the weights is more important than the size of the network.
\newblock {\em IEEE transactions on Information Theory}, 44(2):525--536, 1998.

\bibitem{bartlett2017spectrally}
Peter~L Bartlett, Dylan~J Foster, and Matus~J Telgarsky.
\newblock Spectrally-normalized margin bounds for neural networks.
\newblock In {\em Advances in Neural Information Processing Systems}, pages
  6241--6250, 2017.

\bibitem{bartlett2002rademacher}
Peter~L Bartlett and Shahar Mendelson.
\newblock Rademacher and gaussian complexities: Risk bounds and structural
  results.
\newblock {\em Journal of Machine Learning Research}, 3:463--482, 2002.

\bibitem{Boucheron2003Concentration}
Stéphane Boucheron, Gábor Lugosi, and Olivier Bousquet.
\newblock Concentration inequalities.
\newblock In {\em Summer School on Machine Learning}, pages 208--240, 2003.

\bibitem{cybenko1989approximation}
George Cybenko.
\newblock Approximation by superpositions of a sigmoidal function.
\newblock {\em Mathematics of control, signals and systems}, 2(4):303--314,
  1989.

\bibitem{Eldan2016ThePO}
Ronen Eldan and Ohad Shamir.
\newblock The power of depth for feedforward neural networks.
\newblock In {\em Conference on Learning Theory}, pages 907--940, 2016.

\bibitem{pmlr-v75-golowich18a}
Noah Golowich, Alexander Rakhlin, and Ohad Shamir.
\newblock Size-independent sample complexity of neural networks.
\newblock In {\em Proceedings of the 31st Conference On Learning Theory}, 2018.

\bibitem{goodfellow2016deep}
Ian Goodfellow, Yoshua Bengio, Aaron Courville, and Yoshua Bengio.
\newblock {\em Deep learning}, volume~1.
\newblock MIT press Cambridge, 2016.

\bibitem{he2016deep}
Kaiming He, Xiangyu Zhang, Shaoqing Ren, and Jian Sun.
\newblock Deep residual learning for image recognition.
\newblock In {\em Proceedings of the IEEE conference on computer vision and
  pattern recognition}, pages 770--778, 2016.

\bibitem{hornik1991approximation}
Kurt Hornik.
\newblock Approximation capabilities of multilayer feedforward networks.
\newblock {\em Neural networks}, 4(2):251--257, 1991.

\bibitem{kakade2009complexity}
Sham~M Kakade, Karthik Sridharan, and Ambuj Tewari.
\newblock On the complexity of linear prediction: Risk bounds, margin bounds,
  and regularization.
\newblock In {\em Advances in Neural Information Processing Systems}, pages
  793--800, 2009.

\bibitem{ledoux2013probability}
Michel Ledoux and Michel Talagrand.
\newblock {\em Probability in Banach Spaces: isoperimetry and processes}.
\newblock Springer Science \& Business Media, 2013.

\bibitem{liang2016deep}
Shiyu Liang and R~Srikant.
\newblock Why deep neural networks for function approximation?
\newblock In {\em International Conference on Learning Representations}, 2017.

\bibitem{mohri2012foundations}
Mehryar Mohri, Afshin Rostamizadeh, and Ameet Talwalkar.
\newblock {\em Foundations of machine learning}.
\newblock MIT press, 2012.

\bibitem{neyshabur2018a}
Behnam Neyshabur, Srinadh Bhojanapalli, and Nathan Srebro.
\newblock A {PAC}-bayesian approach to spectrally-normalized margin bounds for
  neural networks.
\newblock In {\em International Conference on Learning Representations}, 2018.

\bibitem{neyshabur2015norm}
Behnam Neyshabur, Ryota Tomioka, and Nathan Srebro.
\newblock Norm-based capacity control in neural networks.
\newblock In {\em Conference on Learning Theory}, pages 1376--1401, 2015.

\bibitem{pinkus1999approximation}
Allan Pinkus.
\newblock Approximation theory of the mlp model in neural networks.
\newblock {\em Acta numerica}, 8:143--195, 1999.

\bibitem{pmlr-v70-safran17a}
Itay Safran and Ohad Shamir.
\newblock Depth-width tradeoffs in approximating natural functions with neural
  networks.
\newblock In {\em Proceedings of the 34th International Conference on Machine
  Learning}, volume~70 of {\em Proceedings of Machine Learning Research}, pages
  2979--2987, International Convention Centre, Sydney, Australia, 2017. PMLR.

\bibitem{salimans2016weight}
Tim Salimans and Diederik~P Kingma.
\newblock Weight normalization: A simple reparameterization to accelerate
  training of deep neural networks.
\newblock In {\em Advances in Neural Information Processing Systems}, pages
  901--909, 2016.

\bibitem{shalev2014understanding}
Shai Shalev-Shwartz and Shai Ben-David.
\newblock {\em Understanding machine learning: From theory to algorithms}.
\newblock Cambridge university press, 2014.

\bibitem{shalev2007primal}
Shai Shalev-Shwartz and Yoram Singer.
\newblock A primal-dual perspective of online learning algorithms.
\newblock {\em Machine Learning}, 69(2-3):115--142, 2007.

\bibitem{sun2016depth}
Shizhao Sun, Wei Chen, Liwei Wang, Xiaoguang Liu, and Tie-Yan Liu.
\newblock On the depth of deep neural networks: A theoretical view.
\newblock In {\em AAAI}, pages 2066--2072, 2016.

\bibitem{pmlr-v49-telgarsky16}
Matus Telgarsky.
\newblock Benefits of depth in neural networks.
\newblock In {\em 29th Annual Conference on Learning Theory}, volume~49 of {\em
  Proceedings of Machine Learning Research}, pages 1517--1539, Columbia
  University, New York, New York, USA, 2016. PMLR.

\bibitem{yarotsky2017error}
Dmitry Yarotsky.
\newblock Error bounds for approximations with deep relu networks.
\newblock {\em Neural Networks}, 94:103--114, 2017.

\bibitem{zhang2004statistical}
Tong Zhang.
\newblock Statistical analysis of some multi-category large margin
  classification methods.
\newblock {\em Journal of Machine Learning Research}, 5:1225--1251, 2004.

\end{thebibliography}
\clearpage
\appendix
\textbf{Supplementary Material}
\section{Claim \ref{claim:1}}
\subsection{Proof for Claim \ref{claim:1}}
\begin{proof}
	We first show that for any $\gamma_0 >0$, any norm $\norm{\cdot}_*$, and  any $C_0 >0$,  there exists a function $f_{\tilde{\VV}}$ satisfying $f_{\tilde{\VV}} \equiv C_0$ and $\prod\limits_{i=1}^{k+1}\norm{\tilde{\VV}_i^T}_*\le \gamma_0$. First assume that $\norm{(1,0,\cdots,0)}_*= a_0$. Note that $a_0> 0$ by the definition of the norm. To prove this, we could set an arbitrary $\tilde{\VV}_1$ satisfying that $\norm{\tilde{\VV}_1^T}_{*}=\frac{\gamma_0}{a_0C_0}$,  the arbitrary $\tilde{\VV}_i$s satisfying that $\norm{\tilde{\VV}_i}_{*}=1$ for $i=2,\cdots,k$, and the output layer as $T_{k+1}(\u)=C_0$. Then $f_{\tilde{\VV}}\equiv C_0$, and \[\prod\limits_{i=1}^{k+1}\norm{\tilde{\VV}_i}_{*}\le \frac{\gamma_0}{a_0C_0}*1^{k-1}*a_0C_0=\gamma_0.\]	
	Then 
	\begin{subequations}
		\begin{align}\Rademacherh_S(\N_{\gamma_{*}\le \gamma}^{k,\dd})&=\E_\epsilon\left[ \sup_{f \in \F}{\left( \frac{1}{n} \sum_{i=1}^n{\epsilon_i f(z_i)} \right)} \right]\nonumber\\
		&\ge
		\P(\sum_{i=1}^n\epsilon_i\ne 0)\E_\epsilon\left[ \sup_{f \in \N_{\gamma_{*}\le \gamma}^{k,\dd}}{\left( \frac{1}{n} \sum_{i=1}^n{\epsilon_i f(z_i)} \right)}|\sum_{i=1}^n\epsilon_i\ne 0 \right] \nonumber\\
		&\ge \frac{1}{2}\E_\epsilon\left[ \sup_{f \in\N_{\gamma_{*}\le \gamma}^{k,\dd}}{\left( \frac{1}{n} \sum_{i=1}^n{\epsilon_i f(z_i)} \right)}|\sum_{i=1}^n\epsilon_i\ne 0 \right]\label{eqn:c1}\\
		&\ge \frac{1}{2}\E_\epsilon\left[ \sup_{C_0> 0}{\left( \frac{1}{n} \sum_{i=1}^n\epsilon_isgn(\sum_{i=1}^n\epsilon_i)C_0 \right)}|\sum_{i=1}^n\epsilon_i\ne 0 \right]\nonumber\\
		&=\infty,\nonumber
		\end{align}
	\end{subequations}
	where the step in Equation \eqref{eqn:c1} follows from $	\P(\sum_{i=1}^n\epsilon_i\ne 0)=1$ when $n$ is an odd number, and $	\P(\sum_{i=1}^n\epsilon_i\ne 0)=1-\frac{1}{2}\P(\sum_{i=2}^n\epsilon_i= 1)-\frac{1}{2}\P(\sum_{i=2}^n\epsilon_i= -1)\ge \frac{1}{2}$ when n is an even number. 
\end{proof}
\section{Theorem \ref{thm:nn:prop}}
\begin{proof}
		For Part \ref{nn:prop:1}, if any $\norm{T_i}_{p,q}=0$, then $f=0\in\N^{k,\dd}_{p,q,c,c_o}$. Otherwise, we will prove by induction on depth $k+1$. It is trivial when $k=0$. 
		
		When $k=1$,  we rescale the first hidden layer by \[s= c/\norm{T_1}_{p,q}.
		\] Equivalently, define the new affine transformation $T_1^*$ by \[\BB_1^*=s\BB_1,   \WW_1^*=s\WW_1,\] such that $\norm{T_1^*}_{p,q}=c$. For the output layer, we define \[\WW_2^*=\WW_2\norm{T_1^*}_{p,q}/c, \BB_2^*=\BB_2.\] Then $T_2^*(\uu)=(\WW_2^*)^T\uu+\BB_2^*$ satisfies $\norm{T_2^*(\uu)}_{p,q}\le c_o$, as $s\ge 1$. What's more $f(\x)=T_2^*\circ \si\circ T_1^*\circ \x \in \N_{p,q,c,c_o}^{1,\dd}$.
		
		Assume the result holds when $k< K$. Then when $k=K$, consider $f(\x)=T_{K+1}\circ \si \circ T_{K}\circ \cdots\si\circ T_1^*\circ \x$. Its $K$th hidden layer \[f_K(\x)\in \N_{p,q,c,c}^{K-1,\dd_{K}} \] by induction assumption, where $\dd_{K}=(d_0,d_1\cdots,d_K)$. In other words, there exists a series of affine transformations $\{T_i^*\}_{i=1,\cdots,K}$, such that  \[f_K(\x)=T^*_{K}\circ \si\circ T^*_{K-1}\circ \cdots\circ \si\circ T_1^*\circ \x,\] $\norm{T^*_i}=c$ for $i=1,\cdots,K-1$, and $\norm{T^*_K}\le c$. Thus \[f(\xx)=T_{K+1}\circ\si\circ T^*_{K}\circ \si\circ T^*_{K-1}\circ \cdots\circ \si \circ T_1^*\circ \x.\] We rescale $T_K^*$ by $s= c/\norm{T_K^*}_{p,q}$. Equivalently, define a new affine transformation $T_K^{**}$ by $T_K^{**}=sT_K^{*}$, such that $\norm{T_K^{**}}_{p,q}=c$. For the output layer, we define \[\WW_{K+1}^*=\WW_{K+1}/s, \BB_{K+1}^*=\BB_{K+1}.\] Then $T_{K+1}^*(\uu)=(\WW_{K+1}^*)^T\uu+\BB_{K+1}^*$ satisfies $\norm{T_{K+1}^*(\uu)}_{p,q}\le c_o$, as $s\ge 1$. Thus $f\in \N_{p,q,c,c_o}^{K,\dd}$.

	For Part \ref{nn:prop:2}, it is a direct conclusion from Part \ref{nn:prop:1} that $ \N_{p,q,c_1,c_o}^{k,\dd}\subseteq \N_{p,q,c_2,c_o}^{k,\dd}$ if $c_1\le c_2$, and  $ \N_{p,q,c,c^1_o}^{k,\dd}\subseteq \N_{p,q,c,c_o^2}^{k,\dd}$ if $c_o^1\le c_o^2$. If $g\in \N^{k,\dd}_{p,q,c,1}$, then by definition, $c_og\in \N^{k,\dd}_{p,q,c,c_o}$.

	For Part \ref{nn:prop:norm}, note that $\norm{\cdot}_{p_1}\ge\norm{\cdot}_{p_2}$ when $p_1\le p_2$, hence \[\{\vv:\norm{\vv}_{p_1}\le C\}\subseteq \{\vv:\norm{\vv}_{p_2}\le C\}.\] Then the first line of Part \ref{nn:prop:norm} follows from the observation above as well as the conclusion of Part \ref{nn:prop:1}. As for the second line, for any $h\in \N_{p,\infty,c,c_o}^{k,\dd}$, we could write \[h=T_{k+1}\circ\si\circ T_{k}\circ\cdots\circ\si\circ T_{1}\circ\x ,\] where $T_i(\uu):\R^{d_{i-1}}\to\R^{d_i}=\W_i^T\uu+\BB_i$, satisfies that $\norm{T_i}_{p,\infty}=c$ for $i=1,\cdots,k$, and $\norm{T_{k+1}}_{p,\infty}\le c_o$. Note that \[\norm{T_i}_{p,\infty}\le \norm{T_i}_{p,q}\le d_i^{\frac{1}{q}}\norm{T_i}_{p,\infty}\le \max^{\frac{1}{q}}(\dd_{-1})\norm{T_i}_{p,\infty}\] for $i=1,2,\cdots,k$, and $\norm{T_{k+1}}_{p,q}\le d_{k+1}^{\frac{1}{q}}\norm{T_{k+1}}_{p,\infty}$. Thus we get the desired result by Part \ref{nn:prop:1}.

	Regarding Part \ref{nn:prop:widendeepen}, we first show the result holds when $k_1=k_2$. For any $g\in \N_{p,q,c,c_o}^{k_1,\dd^1}$, we could add $d^2_i-d^1_i$ neurons in each hidden layer with no connection to other neurons, thus not increasing the norm of each layer. Note that this neural network belongs to $\N_{p,q,c,c_o}^{k_1,\dd^2}$.  
	
	For the general case when $k_1\le k_2$, we could add $k_2-k_1$ identity layers of width 1 with their $L_{p,q}$ norm equals $1\le c$. Then the new neural network represents the same function as the original one. Combining the conclusion of Part \ref{nn:prop:1}, we have \[\N_{p,q,c,c_o}^{k_1,\dd^1}\subseteq \N_{p,q,c,c_o}^{k_2,\tilde{\dd}^1},\] where $\tilde{d}^1_i= d_i^1$ for $i=0,1,\cdots,k_1$, and $\tilde{d}_i^1= d_{k_1+1}^1$ for  $i=k_1+1,\cdots,k_2+1$. Note that $\N_{p,q,c,c_o}^{k_2,\tilde{\dd}^1}\subseteq \N_{p,q,c,c_o}^{k_2,{\dd}^2}$ by the case when $k_1=k_2$. Thus we get what is expected. 
\end{proof}
\section{Radermacher Complexities}
Rademacher complexity is commonly used to measure the complexity of a hypothesis class with respect to a probability distribution or a sample and analyze generalization bounds \cite{bartlett2002rademacher}. 
\paragraph{Rademacher Complexities.}
The \emph{empirical Rademacher complexity} of the hypothesis class $\F$ with respect to a data set ${S = \{z_1 \dots z_n\}}$ is defined as:
\begin{align*}
\Rademacherh_S(\F) = \E_\epsilon\left[ \sup_{f \in \F}{\left( \frac{1}{n} \sum_{i=1}^n{\epsilon_i f(z_i)} \right)} \right]
\end{align*}
\noindent where ${\epsilon = \{\epsilon_1 \dots \epsilon_n\}}$ are $n$ independent Rademacher random variables. The  \emph{Rademacher complexity} of the hypothesis class $\F$ with respect to $n$ samples is defined as:
\begin{align*}
\Rademacher_n(\F) = \E_{S \sim \D^n}\left[\Rademacherh_S(\F)\right]
\end{align*}

We list the following technical lemmas that will be used later in our own proofs for reference. 
\begin{lemma} \label{lemma:r-elem}
	Let $\F$ and $\G$ be two hypothesis classes and ${a \in \R}$ be a constant.
	Define the shorthand notation:
	\begin{align*}
	& a\F  = \{a f \mid f \in \F\} \\
	& \F+\mathcal{G} = \{f+g \mid f \in \F \text{\ and\ } g \in \G\}
	\end{align*}
	We have:
	\begin{align*}
	{\rm i.\ } & \Rademacherh_S(a\F) = |a| {\rm\ } \Rademacherh_S(\F) \\
	{\rm ii.\ } & \F \subseteq \G {\rm\ } \Rightarrow {\rm\ } \Rademacherh_S(\F) \leq \Rademacherh_S(\G) \\
	{\rm iii.\ } & \Rademacherh_S(\F+\G) \leq \Rademacherh_S(\F) + \Rademacherh_S(\G)
	\end{align*}
\end{lemma}
\begin{proof}
By definition. 
	\end{proof}
\begin{lemma}\cite{ledoux2013probability} \label{lem:ledouxtalagrand}
	Assume that the hypothesis class $\F \subseteq \{f|f:\X\to \R\} $ and $\x_1,\cdots,\x_n\in\X$. Let $G : \R \to \R$ be convex and increasing.
	Assume that the function ${\phi : \R \to \R}$ is $L$-Lipschitz continuous and satisfies that $\phi(0)=0$.
	We have:
	\begin{align*}
	 \E_\epsilon\left[G\left( \sup_{f \in \F}{\left( \frac{1}{n} \sum_{i=1}^n{\epsilon_i \phi(f(\x_i))} \right)} \right)\right]\leq  \E_\epsilon\left[G\left( L\sup_{f \in \F}{\left( \frac{1}{n} \sum_{i=1}^n{\epsilon_i f(\x_i)} \right)} \right)\right]
	\end{align*}
\end{lemma}

\begin{lemma}[Massart's finite lemma]\label{lemma:massart}
Let $\mathcal{A}$ be some finite subset of $\R^m$ and $\epsilon_1,\epsilon_2,\cdots,\epsilon_m$ be independent Radermacher random variables. Let $r=\sup_{\aa\in\mathcal{A}}\norm{\aa}_2$, then we have 
\[\E\left[ \sup_{\aa\in\mathcal{A}}\frac{1}{m}\sum\limits_{i=1}^m\epsilon_ia_i\right]=\frac{r\sqrt{2\log|\mathcal{A}|}}{m}\]
\end{lemma}
The theorem below is a more general version of \cite[Theorem 3.1]{mohri2012foundations}, where they assume $a=0$, of which the proof is very similar to the original one.
\begin{theorem}\label{thm:radunifconv}
		Let $z$ be a random variable of support $\Z$ and distribution $\D$.
		Let ${S = \{z_1 \dots z_n\}}$ be a data set of $n$ i.i.d. samples drawn from $\D$.
		Let $\F$ be a hypothesis class satisfying ${\F \subseteq \{f \mid f : \Z \to [a,a+1]\}}$.
		Fix ${\delta \in (0,1)}$.
		With probability at least ${1-\delta}$ over the choice of $S$, the following holds for all	$h \in \F$:
\[\E_{\D} [h] \leq \Eh_{S} [h] + 2 \Rademacher_{n}(\F) + \sqrt{\frac{\log{(1/\delta)}}{2n}}
\]

	\end{theorem}

\section{Propositions \ref{prop:RCp1}, \ref{prop:RCp2}, \ref{prop:lower}}
In this section, define $\si(u)=uI\{u>0\}$ for $u\in\R$ and $\si\circ \zz=(\si(z_1),\cdots,\si(z_m))$ for any vector $\zz\in\R^m$. 

\subsection{Proof for Proposition \ref{prop:RCp1}}
\begin{proof}
	By Theorem \ref{thm:nn:prop}, $\N_{1,q,c,c_o}^{k,\dd}\subseteq \N_{1,\infty,c,c_o}^{k,\dd}$. Therefore it is sufficient to show that the result holds for $\N_{1,\infty,c,c_o}^{k,\dd}$.  
	
	 In order to get the first term inside the minimum operator, we will show that $\N_{1,\infty,c,c_o}^{k,\dd}$ belongs to some DNN class with only bias neuron in the input layer. Then the result follows from Theorem 2\cite{pmlr-v75-golowich18a}. Define $\N_{\gamma_{1,\infty}\le \gamma}^{k,\dd^+}$ as a function class that contains all functions representable by $f=T_{k+1}\circ\si\circ T_{k}\circ\cdots\circ\si\circ T_{1}\circ\x$  satisfying that \[\gamma_{1,\infty}=\prod\limits_{i=1}^{k+1}\norm{\WW_i}_{1,\infty}\le \gamma,\] where $\dd^+=(m_1+1,d_1+1,d_2+1,\cdots,d_k+1,1)$, $T_i(\uu)=\W_i^T\uu$,  and $\W_{i}\in \R^{d^+_{i-1}\times d^+_{i}}$ for $i=1,\cdots,k+1$.  
	
	The next step is to prove that $\N_{1,\infty,c,c_o}^{k,\dd}\subseteq \N^{k,\dd^+}_{\gamma_{1,\infty}\le \max(1,c)^kc_o}$. Following the notations in Section 2, for any $\tilde{\VV}_i\in \R^{(d_{i-1}+1)\times d_i}$ satisfying that $\norm{\tilde{\VV}_i}_{1,\infty}=c$, we have $\norm{\VV_i}_{1,\infty}=\max(1,c)$, where $\VV_i=(\ee_{1i},\tilde{\VV}_i)$ and $\ee_{1i}=(1,0,\cdots,0)^T\in \R^{d_{i-1}+1}$. Equivalently, the bias neuron in the $i$th hidden layer can be regarded as a hidden neuron computed from the $i-1$th layer by $\si(\ee_{1i}^Tf^*_{i-1}(\x))=1$, while the new affine transformation could be parameterized by $\VV_i$, such that $\norm{\VV_i}_{1,\infty}=\max(1,c)$.  
	
	Finally, we get the first term inside the minimum operator by applying Theorem 2\cite{pmlr-v75-golowich18a}, and the second term is the bound of Proposition \ref{prop:RCp2} when $p=1$.
\end{proof}
\subsection{Proposition \ref{prop:RCp2}}
We first introduce two technical lemmas, which will be used later to prove Proposition \ref{prop:RCp2}.
\begin{lemma} \label{lemma:rader:linear}
	$\zz_i\in \R^{m_1}, \norm{\zz_i}_{\infty}\le 1$ for $i=1,2,\cdots, n$. For $p\in(1,2]$, 
	\[\frac{1}{n}\E\norm{\sum\limits_{i=1}^n\epsilon_i\zz_i}_{p^{*}}\le \frac{m_1^{\frac{1}{p^{*}}}}{\sqrt{n}}\min\left((\sqrt{p^{*}-1}, \sqrt{2\log(2m_1)} \right) ,\]
	and for $p=1\cup(2,\infty)$, 
	\[\frac{1}{n}\E\norm{\sum\limits_{i=1}^n\epsilon_i\zz_i}_{p^{*}}\le \sqrt{\frac{2\log(2m_1)}{n}}m_1^{\frac{1}{p^{*}}}.\]
\end{lemma}
\begin{lemma}\label{lemma:rader:p2}
	$\forall p,q\ge 1$, $s_1, s_2\ge 1$, $\epsilon\in\{-1,+1\}^n$  and for all functions $g: \R^{m_1}\to\R^{s_1}$,  we have 
	\[
	\sup_{\VV\in \R^{s_1\times s_2 }}\frac{1}{\norm{\VV}_{p,q}}\norm{\sum\limits_{i=1}^n\epsilon_i \si\circ \left(\VV^Tg(\x_i) \right)}_{p^{*}}=s_2^{[\frac{1}{p^{*}}-\frac{1}{q}]_+}  \sup_{\vv\in \R^{s_1}}\frac{1}{\norm{\vv}_{p}}\left\lvert\sum\limits_{i=1}^n\epsilon_i\si\left( \dotprod{\vv}{g(\x_i)}\right)\right\rvert ,
	\]where $\frac{1}{p}+\frac{1}{p^{*}}=1$.
\end{lemma}

\subsection{Proof of Proposition \ref{prop:RCp2}}
\begin{proof}	
The proof has two main steps. 

Fixing the sample $S$, $p\ge1$ and the architecture of the DNN, define a series of random variables $\{Z_0,Z_1,\cdots,Z_{k}\}$ as \[Z_{0}=\norm{\sum\limits_{i=1}^n\epsilon_i\x_i}_{p^*}\] and  \[Z_{j}=\sup\limits_{f\in\N^{k,\dd}_{p,q,c,c_o}}\norm{\sum\limits_{i=1}^n\epsilon_i\si\circ f_j(\x_i)}_{p^*},\]
for $j=1,\cdots,k$, where $\{\epsilon_1,\cdots,\epsilon_n\}$ are $n$ independent Rademacher random variables, and $f_j$ denotes the $j$th hidden layer of the WN-DNN $f$. 

		In the first step, we prove by induction that for $j=1,\cdots,k$ and any $t\in\R$ \[\E_\epsilon\exp(tZ_j)\le4^{j} \exp\left(\frac{t^2ns^2_j}{2}+tc^j\prod\limits_{i=1}^jd_i^{[\frac{1}{p^*}-\frac{1}{q}]_+}A_{m_1,S}^{p}\right),\]
		where  \[s_j=\sum\limits_{i=2}^jc^{j-i+1}\prod\limits_{l=i}^jd_l^{[\frac{1}{p^*}-\frac{1}{q}]_+}+(m_1^{\frac{1}{p^*}}+1)c^j\prod\limits_{l=1}^jd_l^{[\frac{1}{p^*}-\frac{1}{q}]_+}\] and 	\begin{equation*}
		A_{m_1,S}^{p} = \left\{
		\begin{array}{lcl}
		\sqrt{n}\min\left((\sqrt{p^{*}-1}m_1^{\frac{1}{p^{*}}}, \sqrt{2\log(2m_1)}m_1^{\frac{1}{p^{*}}} \right) &\text{if} &p\in(1,2]\\
		\sqrt{2n\log(2m_1)}m_1^{\frac{1}{p^{*}}}&\text{if} &p\in1\cup(2,\infty)
		\end{array}  
		\right.
		\end {equation*}
		Note that $s_{j+1}=cd_j^{[\frac{1}{p^*}-\frac{1}{q}]_+}(s_j+1)$.
		
	When $j= 0$, by Lemma \ref{lemma:rader:linear}, $E_\epsilon Z_0\le A_{m_1,S}^{p}$. Note that $Z_0$ is a deterministic function of the i.i.d.random variables $\epsilon_1,\cdots,\epsilon_n$, satisfying that
	\[\lvert Z_0(\epsilon_1,\cdots,\epsilon_i,\cdots,\epsilon_n)-Z_0(\epsilon_1,\cdots,-\epsilon_i,\cdots,\epsilon_n)\rvert\le 2\max\norm{\xx_i}_{p^*}\le 2m_1^{\frac{1}{p^{*}}}\] by Minkowski inequality. By the proof of Theorem 6.2 \cite{Boucheron2003Concentration}, $Z_0$ satisfies that $\log\E_\epsilon\exp\left(t(Z_0-E_\epsilon Z_0)\right)\le t^2nm_1^{\frac{2}{p^{*}}}/2$, thus 
\begingroup        
\allowdisplaybreaks
\begin{align*} 
\E_\epsilon\exp\left(tZ_0\right) &= \E_\epsilon\exp\left(t(Z_0-E_\epsilon Z_0)\right)*\exp\left(tE_\epsilon Z_0\right)\\
&\le \exp\left(\frac{t^2nm_1^{\frac{2}{p^{*}}}}{2}+tA_{m_1,S}^{p}\right)
\end{align*}
\endgroup
	for any $t\in\R$.

For the case when $j=1,\cdots,k$, 
	\begingroup        
\allowdisplaybreaks
\begin{subequations}
\begin{align} 
\E_\epsilon\exp\left(tZ_j\right)& = 	\E_\epsilon\exp\left(t\sup_{\norm{\tilde{\VV}_j}_{p,q}\le c}\norm{\sum\limits_{i=1}^n\epsilon_i\si\circ\left(\tilde{\VV}_j^T\si\circ f^*_{j-1}(\x_i)\right)}_{p^*}\right) \nonumber \\
&=\E_\epsilon\exp\left(tcd_j^{[\frac{1}{p^*}-\frac{1}{q}]_+}\sup_{\vv,f}\left\lvert\sum\limits_{i=1}^n\epsilon_i\si(\vv^T\si\circ f^*_{j-1}(\x_i))/\norm{\vv}_p\right\rvert\right)\label{eqn:prop1:k1:1}\\
&\le 2\E_\epsilon\exp\left(tcd_j^{[\frac{1}{p^*}-\frac{1}{q}]_+}\sup_{\vv,f}\sum\limits_{i=1}^n\epsilon_i\si(\vv^T\si\circ f^*_{j-1}(\x_i))/\norm{\vv}_p\right)\label{eqn:prop1:k1:2}\\
&\le 2\E_\epsilon\exp\left(tcd_j^{[\frac{1}{p^*}-\frac{1}{q}]_+}\sup_{\vv,f}\vv^T\sum\limits_{i=1}^n\epsilon_i\si\circ f^*_{j-1}(\x_i)/\norm{\vv}_p\right)\label{eqn:prop1:k1:3}\\
&\le 2\E_\epsilon\exp\left(tcd_j^{[\frac{1}{p^*}-\frac{1}{q}]_+}\sup_{f}\norm{\sum\limits_{i=1}^n\epsilon_i(1,\si\circ f_{j-1}(\x_i))}_{p^*}\right)\nonumber\\
&\le 2\E_\epsilon\exp\left(tcd_j^{[\frac{1}{p^*}-\frac{1}{q}]_+}(\lvert\sum\limits_{i=1}^n\epsilon_i\rvert+\sup_f\norm{\sum\limits_{i=1}^n\epsilon_i\si\circ f_{j-1}(\x_i)}_{p^*})\right)\nonumber\\
&\le 2\left[ \E_\epsilon\exp\left(r_jtcd_j^{[\frac{1}{p^*}-\frac{1}{q}]_+}\lvert\sum\limits_{i=1}^n\epsilon_i\rvert\right)\right]^{\frac{1}{r_j}}\left[\E_\epsilon\exp\left(r_j^*tcd_j^{[\frac{1}{p^*}-\frac{1}{q}]_+}\sup_f\norm{\sum\limits_{i=1}^n\epsilon_i\si\circ f_{j-1}(\x_i)}_{p^*}\right)\right]^{\frac{1}{r_j^*}}\label{eqn:prop1:k1:holder}\\
&\le 2\left[ 2\E_\epsilon\exp\left(r_jtcd_j^{[\frac{1}{p^*}-\frac{1}{q}]_+}\sum\limits_{i=1}^n\epsilon_i\right)\right]^{\frac{1}{r_j}}\left[\E_\epsilon\exp\left(r_j^*tcd_j^{[\frac{1}{p^*}-\frac{1}{q}]_+}Z_{j-1}\right)\right]^{\frac{1}{r_j^*}}, \label{eqn:prop1:k1:f}\\
&\le 4^{1+\frac{j-1}{r_j^*}}\exp\left(\frac{nt^2c^2d_j^{2[\frac{1}{p^*}-\frac{1}{q}]_+}(1+s_{j-1})^2}{2}+tc^j\prod\limits_{i=1}^jd_i^{[\frac{1}{p^*}-\frac{1}{q}]_+}A_{m_1,S}^{p}\right)\nonumber\\
&\le 4^{j}\exp\left(\frac{nt^2s_j^2}{2}+tc^j\prod\limits_{i=1}^jd_i^{[\frac{1}{p^*}-\frac{1}{q}]_+}A_{m_1,S}^{p}\right)\nonumber
\end{align}
\end{subequations}
\endgroup
The step in Equation \eqref{eqn:prop1:k1:1} follows from Lemma \ref{lemma:rader:p2}. The step in Equation \eqref{eqn:prop1:k1:2} follows from the observation that \begin{align*}\E_\epsilon\exp\left(\sup_{\vv}\left\lvert\sum\limits_{i=1}^n\epsilon_i\frac{\si(\vv^Tf^*_{j-1}(\x_i))}{\norm{\vv}_p}\right\rvert\right)&\le\E_\epsilon\exp\left(\sup_{\vv}\sum\limits_{i=1}^n\epsilon_i\frac{\si(\vv^Tf^*_{j-1}(\x_i))}{\norm{\vv}_p}\right)+\\\E_\epsilon\exp\left(\sup_{\vv}\sum\limits_{i=1}^n(-\epsilon_i)\frac{\si(\vv^Tf^*_{j-1}(\x_i))}{\norm{\vv}_p}\right)
&=2\E_\epsilon\exp\left(\sup_{\vv}\sum\limits_{i=1}^n\epsilon_i\frac{\si(\vv^Tf^*_{j-1}(\x_i))}{\norm{\vv}_p}\right). \end{align*}
The step in Equation \eqref{eqn:prop1:k1:3} follows from Lemma \ref{lem:ledouxtalagrand}. Note that Equation \eqref{eqn:prop1:k1:holder} holds for any $r>1$ and $r^*=\frac{r}{r-1}$ by H\"older's inequality $\E(|XY|)\le \E(|X|^{r})^{\frac{1}{r}}\E(|Y|^{r^*})^{\frac{1}{r^*}}$. An optimal $r_j=s_{j-1}+1$ is chosen in our case. The step in Equation \eqref{eqn:prop1:k1:f} follows from $\E_\epsilon\exp\left(\lvert X\rvert\right)\le \E_\epsilon\exp\left(X\right)+\E_\epsilon\exp\left(-X\right)$.

Note that $\sum\limits_{i=1}^n\epsilon_i$ is also a deterministic function of the i.i.d.random variables $\epsilon_1,\cdots,\epsilon_n$, satisfying that  $\E_\epsilon \sum\limits_{i=1}^n \epsilon_i=0$ and
\[\lvert\sum\limits_{i\ne j}\epsilon_i+\epsilon_j-(\sum\limits_{i\ne j}\epsilon_i-\epsilon_j)\rvert\le 2.\] Then by the proof of Theorem 6.2 \cite{Boucheron2003Concentration}, \[\E_\epsilon \exp(t\sum\limits_{i=1}^n\epsilon_i)\le \exp(\frac{t^2n}{2})\]  for any $t\in\R$.
Then we get the desired result by choosing the optimal $r_j$ while following the induction assumption.

The second step is based on the idea of \cite{pmlr-v75-golowich18a} using Jensen's inequality. For any $\lambda>0$,
	\begingroup        
\allowdisplaybreaks
\begin{subequations}
\begin{align} 
	n\Rademacherh_S(\N_{p,q,c,c_o}^{k,\dd})&=\E_\epsilon\left[ \sup_{f \in \N_{p,q,c,c_o}^{k,\dd}}{\left( \sum_{i=1}^n{\epsilon_i f(\x_i)}\right)} \right]\nonumber\\
	&\le \frac{1}{\lambda}\log\E_\epsilon \exp\left(\lambda\sup_{f \in \N_{p,q,c,c_o}^{k,\dd}}{\left(  \sum_{i=1}^n{\epsilon_i f(\x_i)}\right)}\right)\nonumber\\
	&\le  \frac{1}{\lambda}\log\E_\epsilon \exp\left(\lambda c_o\sup_{f \in \N_{p,q,c,c_o}^{k,\dd}}\norm{  \sum_{i=1}^n{\epsilon_i (1,\si\circ f_k(\x_i)})}_{p^*}\right)\nonumber\\
&\le \frac{1}{\lambda}\left[(k+1)\log 4+\frac{\lambda^2c_o^2n(s_k+1)^2}{2}+\lambda A_{m_1,S}^{p}c_oc^k\prod\limits_{i=1}^kd_i^{[\frac{1}{p^*}-\frac{1}{q}]_+}\right]\label{eqn:prop2:last}\\
	&=\frac{(k+1)\log 4}{\lambda}+\frac{\lambda c_o^2n(s_k+1)^2}{2}+c_oc^{k}\prod\limits_{i=1}^kd_i^{[\frac{1}{p^*}-\frac{1}{q}]_+}A_{m_1,S}^{p},\nonumber
	\end{align}
	\end{subequations}
	\endgroup
	where the step in Equation \eqref{eqn:prop2:last} is derived using a similar techinique as in Equations \eqref{eqn:prop1:k1:1} to \eqref{eqn:prop1:k1:f}
	By choosing the optimal $\lambda =\frac{\sqrt{(k+1)\log 16}}{c_o(s_k+1)\sqrt{n}} $,
	we have 
	\begin{align*}
	\Rademacherh_S(\N_{p,q,c}^{k,\dd})&\le c_o\sqrt{\frac{(k+1)\log 16}{n}}\left(\sum\limits_{i=2}^kc^{k-i+1}\prod\limits_{l=i}^kd_l^{[\frac{1}{p^*}-\frac{1}{q}]_+}+
	(m_1^{\frac{1}{p^{*}}}+1)c^{k}\prod\limits_{i=1}^kd_i^{[\frac{1}{p^*}-\frac{1}{q}]_+}+1\right)+\\&\frac{1}{\sqrt{n}}c_oc^{k}\prod\limits_{i=1}^kd_i^{[\frac{1}{p^*}-\frac{1}{q}]_+}A_{m_1,S}^{p}
	\end{align*}
\end{proof}

\subsection{Proof of Lemma \ref{lemma:rader:linear}}
\begin{proof}
	For $p\in(1,2]$, or equivalently  $p^{*}\in [2,\infty)$, $\norm{\cdotp}_{p^{*}}$ is $2(p^{*}-1)$-strongly convex with respect to itself on $\R^{m_1+1}$ \cite{shalev2007primal} and $\norm{\zz_i}_{p^{*}}\le m_1^{\frac{1}{p^{*}}}\norm{\zz_i}_\infty$, thus $\frac{1}{n}\E\norm{\sum\limits_{i=1}^n\epsilon_i\zz_i}_{p^{*}}\le \sqrt{\frac{p^{*}-1}{n}} m_1^{\frac{1}{p^{*}}}$ \cite{kakade2009complexity}.
	
	For $p\in [1,\infty)$ or equivalently $p^{*}\in (1,\infty]$, let $\zz[j]=(\zz_1[j],\zz_2[j],\cdots, \zz_n[j])^T$, where $\zz_i[j]$ is the $j$th element of the vector $\zz_i\in \R^{m_1}$.
	\begingroup
	\allowdisplaybreaks
	\begin{align}
	\frac{1}{n}\E\norm{\sum\limits_{i=1}^n\epsilon_i\zz_i}_{p^{*}}&\le 	\frac{m_1^{\frac{1}{p^{*}}}}{n}\E\norm{\sum\limits_{i=1}^n\epsilon_i\zz_i}_{\infty}\nonumber\\
	&\le\frac{m_1^{\frac{1}{p^{*}}}\sqrt{2\log(2m_1)}}{n}\sup_{j}\norm{\zz[j]}_2\label{eqn:lemma:massart}\\
	&\le\frac{m_1^{\frac{1}{p^{*}}}\sqrt{2\log(2m_1)}}{n}\sqrt{n}\sup_{j}\norm{\zz[j]}_\infty\nonumber\\
	&\le \frac{m_1^{\frac{1}{p^{*}}}\sqrt{2\log(2m_1)}}{\sqrt{n}}\nonumber
	\end{align}
	\endgroup
	The step in Equation \eqref{eqn:lemma:massart} follows from Lemma \ref{lemma:massart}.
\end{proof}

\subsection{Proof of Lemma \ref{lemma:rader:p2}}
\begin{proof}
	The proof is based on the ideas of \cite[Lemma 17]{neyshabur2015norm}
	
	The right hand side (RHS) is always less than or equal to the left hand side (LHS), since given any vector $\vv$ we could create a corresponding matrix $\VV$ of which each row is $\vv$. 
	
	Then we will show that (LHS) is always less than or equal to (RHS). Let $\VV[,j]$ be the $j$th column of the matrix $\VV$. We have $\norm{\VV}_{p,p^{*}}\le \norm{\VV}_{p,q}$ when $q\le p^{*}$ and by H\"older's inequality, $\norm{\VV}_{p,p^{*}}\le s_2^{[\frac{1}{p^{*}}-\frac{1}{q}]}\norm{\VV}_{p,q}$ when $q>p^{*}$. Thus 
	\begingroup
	\allowdisplaybreaks
	\begin{align*}
	\rm{(LHS)}&\le \sup_{\VV\in \R^{s_1\times s_2 }}\frac{s_2^{[\frac{1}{p^{*}}-\frac{1}{q}]_+} }{\norm{\VV}_{p,p^{*}}}\norm{\sum\limits_{i=1}^n\epsilon_i \si\circ\left(\VV^Tg(\x_i)\right) }_{p^{*}}\\
	&=s_2^{[\frac{1}{p^{*}}-\frac{1}{q}]_+}\sup_{\VV\in \R^{s_1\times s_2
	}}\frac{1 }{\norm{\VV}_{p,p^{*}}}\left(\sum\limits_{j=1}^{s_2}\left\lvert\sum\limits_{i=1}^n\epsilon_i \si\left(\dotprod{\VV[,j]}{g(\x_i)}\right)\right\rvert^{p^{*}}\right)^{1/p^{*}}\\
	&\le s_2^{[\frac{1}{p^{*}}-\frac{1}{q}]_+}\sup_{\VV\in \R^{s_1\times s_2 }}\frac{1 }{\norm{\VV}_{p,p^{*}}}\left(\sum\limits_{j=1}^{s_2}\left(\norm{\VV[,j]}_p\frac{\rm{(RHS)}}{s_2^{[\frac{1}{p^{*}}-\frac{1}{q}]_+}}\right)^{p^{*}}\right)^{1/p^{*}}\\
	&=\rm{(RHS)}\sup_{\VV\in \R^{s_1\times s_2 }}\frac{1 }{\norm{\VV}_{p,p^{*}}}\left(\sum\limits_{j=1}^{s_2}(\norm{\VV[,j]}_p)^{p^{*}}\right)^{1/p^{*}}\\
	&=\rm{(RHS)}
	\end{align*}
	\endgroup
\end{proof}

\subsection{Proposition \ref{prop:lower}}
\begin{proof}
	Define $\N_{\gamma_{p,q}\le \gamma}^{k,\dd}$ as a function class that contains all functions representable by some neural network $f=T_{k+1}\circ\si\circ T_{k}\circ\cdots\circ\si\circ T_{1}\circ\x$  satisfying that \[\gamma_{p,q}=\prod\limits_{i=1}^{k+1}\norm{\WW_i}_{p,q}\le \gamma,\] where $\dd=(m_1,d,\cdots,d,1)$, $T_i(\uu)=\W_i^T\uu$,  and $\W_{i}\in \R^{d_{i-1}\times d_{i}}$ for $i=1,\cdots,k+1$. In order to use the conclusion of \cite[Theorem 3]{neyshabur2015norm} for DNNs with no bias neuron, it is sufficient to show that \[\N_{\gamma_{p,q}\le \gamma}^{k,\dd}\subseteq \N_{p,q,c,c_o}^{k,\dd},\] for any $c,c_o$ satisfying that $c^kc_o\ge  \gamma$.
	
		If any $\norm{T_i}_{p,q}=0$, then $f=0\in\N^{k,\dd}_{p,q,c,c_o}$. Otherwise, for any $c,c_o$ satisfying that $c^kc_o\ge\gamma\ge \prod\limits_{i=1}^{k+1}\norm{T_i}_{p,q}$,
	we rescale each hidden layer by \[s_i= c/\norm{T_i}_{p,q},\] that is, define $T_i^*$ by $\BB_i^*=0$ and  $\WW_i^*=s_i\WW_i$, such that $\norm{T_i^*}_{p,q}=c$ and $T_i^*=s_iT_i$. Correspondingly, rescale the output layer by $1/\prod\limits_{i=1}^ks_i$ and $\norm{T_{k+1}^*}_{p,q}\le c_o$ as $s_i\ge 1$. Therefore, $f\in\N_{p,q,c,c_o}^{k,\dd}$. 	
	\end{proof}
\section{Generalization Bounds}
In this section, we provide a generalization bound that holds for any data distribution for regression as an extension of Section  3. 
\paragraph{The Regression Problem.}
Assume that $(\x_1,y_1),\dots,(\x_n,y_n)$ are $n$ i.i.d samples on $\mathcal{X}\times\mathcal{Y}\subseteq\mathbb{R}^{m_1}\times\R$, satisfying that 
\begin{equation}
\label{general}
y_i=f(\x_i)+\varepsilon_i,
\end{equation}
where $f: \mathcal{X} \to \Y\subseteq \R$ is an unknown function and $\varepsilon_i$ an independent noise. 
\subsection{Generalization Bounds}
Assume that $d : \Y\times \Y \to [0,1]$ is a 1-Lipschitz function related to the prediction problem. For example, we could define $d(y,y') = \min(1, (y-y')^2/2)$. Let $\zz=(\x,y)\in \Z$, where $\Z=\X\times \Y$. Furthermore, for each $f\in \mathcal{N}_{p,q,c,c_o}^{k,\dd}$, define a corresponding $h_f$ such that $h_f(\zz)= d(y,f(\x))$. Let $\mathcal{H}_{p,q,c,c_o}^{k,\dd}$ be a hypothesis class satisfying 
\begin{equation*}
\mathcal{H}_{p,q,c,c_o}^{k,\dd}=\bigcup\limits_{f\in \mathcal{N}_{p,q,c,c_o}^{k,\dd}}h_f.
\end{equation*}

For every $h \in\mathcal{H}_{p,q,c,c_o}^{k,\dd}$, define the true and empirical risks as 
\begin{equation*}
\E_{\D} [ h ]  =\E_{\zz\sim \D} [ h(\zz) ],  \quad
\Eh_{S}[ h ]  =\frac{1}{n}\sum\limits_{i=1}^nh(\zz_i).
\end{equation*}
\begin{theorem} \label{thm:radunifconv:reg:s1}
	Let $\zz=(\x,y)$ be a random variable of support $\Z$ and distribution $\D$.
	Let ${S = \{\zz_1 \dots \zz_n\}}$ be a dataset of $n$ i.i.d. samples drawn from $\D$.
	Fix ${\delta \in (0,1)}$, $k\in [0,\infty)$ and $d_i\in\mathbb{N}_+$ for $i=1,\cdots,k$. With probability at least ${1-\delta}$ over the choice of $S$, 
	\begin{enumerate}[label=(\alph*)]
		\item \label{a} for $p=1$ and $q\in [1,\infty]$, we have $\forall h \in \mathcal{H}_{1,q,c,c_o}^{k,\dd}$:
		\begin{align*}
		\E_{\D} [ h] &\le \Eh_S [ h] +\sqrt{\frac{\log(1/\delta)}{2n}}+\frac{2c_o}{\sqrt{n}}*\min\left( 2\max(1,c^k)\sqrt{k+2+\log(m_1+1)},\right.\\&\left.\sqrt{(k+1)\log 16}\sum\limits_{i=0}^kc^i+c^k(\sqrt{2\log(2m_1)}+\sqrt{(k+1)\log 16})\right)
		\end{align*}
		\item \label{b}	for $p\in(1,2]$ and $q\in [1,\infty]$,  we have $\forall h \in \mathcal{H}_{p,q,c,c_o}^{k,\dd}$:
		\begin{align*}
		\E_{\D} [ h] &\le \Eh_S [ h] + \sqrt{\frac{\log(1/\delta)}{2n}}+\frac{1}{\sqrt{n}}c_oc^{k}\prod\limits_{i=1}^kd_i^{[\frac{1}{p^*}-\frac{1}{q}]_+}\sqrt{2\log(2m_1)}m_1^{\frac{1}{p^{*}}}+\\&	c_o\sqrt{\frac{(k+1)\log 16}{n}}\left(\sum\limits_{i=1}^{k+1}c^{k-i+1}\prod\limits_{l=i}^kd_l^{[\frac{1}{p^*}-\frac{1}{q}]_+}+
		m_1^{\frac{1}{p^*}}c^{k}\prod\limits_{i=1}^kd_i^{[\frac{1}{p^*}-\frac{1}{q}]_+}\right).
		\end{align*}
		\item \label{c} for $p\in(2,\infty)$ and $q\in [1,\infty]$,  we have $\forall h \in \mathcal{H}_{p,q,c,c_o}^{k,\dd}$:
		\begin{align*}
		\E_{\D} [ h] &\le \Eh_S [ h] + \sqrt{\frac{\log(1/\delta)}{2n}}+\\&\frac{1}{\sqrt{n}}c_oc^{k}\prod\limits_{i=1}^kd_i^{[\frac{1}{p^*}-\frac{1}{q}]_+}m_1^{\frac{1}{p^{*}}}\min\left((\sqrt{p^{*}-1}, \sqrt{2\log(2m_1)} \right)+\\
		&c_o\sqrt{\frac{(k+1)\log 16}{n}}\left(\sum\limits_{i=1}^{k+1}c^{k-i+1}\prod\limits_{l=i}^kd_l^{[\frac{1}{p^*}-\frac{1}{q}]_+}+
		m_1^{\frac{1}{p^*}}c^{k}\prod\limits_{i=1}^kd_i^{[\frac{1}{p^*}-\frac{1}{q}]_+}\right).
		\end{align*}
	\end{enumerate}
\end{theorem}
The corollary below gives a generalization bound for the $L_{1,\infty}$ WN-DNNs. 

\begin{corollary}
	Let $\zz=(\x,y)$ be a random variable of support $\Z$ and distribution $\D$.
	Let ${S = \{\zz_1 \dots \zz_n\}}$ be a dataset of $n$ i.i.d. samples drawn from $\D$.
	Fix ${\delta \in (0,1)}$, $k\in [0,\infty)$  and  $d_i\in\mathbb{N}_+$ for $i=1,\cdots,k$. Assume that $c^k\le a_0$ for some $a_0\ge 1$.  With probability at least ${1-\delta}$ over the choice of $S$, 	for any $h \in \mathcal{H}_{p,q,c,c_o}^{k,\dd}$, we have:
		\begin{align*}
	\E_{\D} [ h]&\le \Eh_S [ h] +\sqrt{\frac{\log(1/\delta)}{2n}}+\frac{4c_oa_0}{\sqrt{n}}\sqrt{k+2+\log(m_1+1)}.
	\end{align*}
\end{corollary}
For instance, we could define $c$ as $1+\frac{v_0}{k}$ with some constant $v_0\ge 0$ for ResNet \cite{he2016deep}, then $c(k)^k\le e^{v_0}$. The case with $v_0=0$ leads to a specific case where the normalization constant $c=1$.
\subsection{Proof of Theorem \ref{thm:radunifconv:reg:s1}}
\begin{proof}
	By applying Theorem \ref{thm:radunifconv}, 
	with probability at least ${1-\delta}$ over the choice of $S$, $\forall h \in \mathcal{H}_{p,q,c,c_o}^{k,\dd}$, we have: \[\E_{\D} [h] -\Eh_{S} [h] \leq2 \Rademacher_{n}(\mathcal{H}_{p,q,c,c_o}^{k,\dd}) + \sqrt{\frac{\log{(1/\delta)}}{2n}}.
	\] 
	Thus it is equivalent to bound $\Rademacher_{n}(\mathcal{H}_{p,q,c,c_o}^{k,\dd})$ in order to bound the absolute value of the generalization error. By Lemma \ref{lem:ledouxtalagrand}, we have: \[\Rademacher_{n}(\mathcal{H}_{p,q,c,c_o}^{k,\dd})\le\Rademacher_{n}(\mathcal{N}_{p,q,c,c_o}^{k,\dd}). \]
	Finally, \ref{a} follows from \[\Rademacher_{n}(\mathcal{N}_{p,q,c,c_o}^{k,\dd})\le \sup_S\Rademacherh_S(\mathcal{N}_{p,q,c,c_o}^{k,\dd})\] and Proposition \ref{prop:RCp1}, while  \[\Rademacher_{n}(\mathcal{N}_{p,q,c,c_o}^{k,\dd})=\Rademacher_{n}(\mathcal{N}_{p,q,c,c_o}^{k,\dd})\le \sup_S\Rademacherh_S(\mathcal{N}_{p,q,c,c_o}^{k,\dd})\] and Proposition \ref{prop:RCp2} lead to \ref{b} and \ref{c}. 
\end{proof}
\section{Theorem \ref{thm:approx}}
\subsection{Proof of Lemma \ref{lemma:approx:con:deep}}
\begin{proof}
$\norm{(b_i,\w_{i}^T)}_1=1$ implies  $\norm{(b_i,2\w_{i}^T)}_1\le2$, thus by Theorem \ref{thm:nn:prop} Part  \ref{nn:prop:2}, it is sufficient to show that  g could be represented by some neural network in $\N_{p,q,wid^{1/q}_k,c_o}^{k,\dd^k}$ if instead  $\norm{(b_i,2\w_{i}^T)}_1=1$. In addition,  by Theorem \ref{thm:nn:prop} Parts \ref{nn:prop:2}, \ref{nn:prop:norm} and \ref{nn:prop:widendeepen}, it is equivalent to show that when $\sum\limits_{i=1}^r |c_i|=1$, g could be represented by some neural network in $\N_{1,\infty,1,1}^{k,\dd}$  where $d_i\le [r/k]+2m_1+3$ for $i=1,\cdots,k$.

Decompose the shallow neural network as
\begin{align*}
g(\x)&=\left(\sum\limits_{i=1}^{r_1}c^+_i \right)g_+(\x)-\left(\sum\limits_{i=1}^{r_2}c^-_i\right)g_-(\x),
\end{align*} where \[g_+(\x)=\sum\limits_{i=1}^{r_1}c_i^+\si\left((\w_i^+)^T\x+b_i^+\right)/\sum\limits_{i=1}^{r_1}c^+_i,\quad  g_-(\x)=\sum\limits_{i=1}^{r_2}c_i^-\si\left((\w_i^-)^T\x+b_i^-\right)/\sum\limits_{i=1}^{r_2}c^-_i\] for some $c_i^+, c_i^->0$. Note that  
$\norm{\alpha^T A^T}_{1}\le 1$ if $\alpha\in \R^s$ satisfies that  $\norm{\alpha}_1\le 1$, and  $A\in \R^{t\times s}$ satisfies that  $\norm{A}_{1,\infty}\le 1$. Additionally \[\sum\limits_{i=1}^{r_1} c^+_i+\sum\limits_{i=1}^{r_2} c^-_i=\sum\limits_{i=1}^r |c_i|=1. \]
Thus it is sufficient to show that \[\left( g_+(\x),g_-(\x) \right)\] could be represented by some neural network in $ \N_{1,\infty,1,1}^{k,\dd}$, where each hidden layer contains both $\si\circ\x$ and $\si\circ(-\x)$, while satisfying that $d_i\le [r_1/k]+[r_2/k]+2m_1+2$  for $i=1,\cdots,k$ and $d_{k+1}=2$.

When $k=1$,  it is trivial.

When $k=2$, we construct the first hidden layer consisting of $[r_1/2]+[r_2/2]+2m_1$ hidden neurons:  \[\{ (\w_i^+)^T\x+b_i^+: i=1,\cdots,[r_1/2] \},\{ (\w_i^-)^T\x+b_i^-: i=1,\cdots,[r_2/2] \}, \x, -\x.\] For the second hidden layer, there are $2+r-([r_1/2]+[r_2/2])+2m_1$ hidden neurons.  The first neuron  \[\eta_1=\sum\limits_{i=1}^{[r_1/2]}c_i^+\si\left((\w_i^+)^T\x+b_i^+\right)/\sum\limits_{i=1}^{[r_1/2]}c^+_i,\] the second neuron  \[\eta_2=\sum\limits_{i=1}^{[r_2/2]}c_i^-\si\left((\w_i^-)^T\x+b_i^-\right)/\sum\limits_{i=1}^{[r_2/2]}c^-_i,\] then follows $\si\circ\x$ , $\si\circ(-\x)$ and the left $r-([r_1/2]+[r_2/2])$ hidden neurons \[\{\eta_i^+=(\w_i^+)^T\si\circ\x-(\w_i^+)^T\si\circ(-\x)+b^+_i: i=[r_1/2]+1,\cdots,r_1\},\]\[\{\eta_i^-=(\w_i^-)^T\si\circ\x-(\w_i^-)^T\si\circ(-\x)+b^-_i: i=[r_2/2]+1,\cdots,r_2\}.\] The output layer only contains two hidden neurons $(g_+,g_-)$, which could be computed respectively by  \begin{align*}&\frac{\sum\limits_{i=1}^{[r_1/2]}c_i^+}{\sum\limits_{i=1}^{r_1}c_i^+}\si(\eta_1)+\sum\limits_{i=[r_1/2]+1}^{r_1}\frac{c_i^+}{\sum\limits_{i=1}^{r_1}c_i^+}\si(\eta_i^+)\quad \text{and}\quad \frac{\sum\limits_{i=1}^{[r_2/2]}c_i^-}{\sum\limits_{i=1}^{r_2}c_i^-}\si(\eta_2)+\sum\limits_{i=[r_2/2]+1}^{r_2}\frac{c_i^-}{\sum\limits_{i=1}^{r_2}c_i^-}\si(\eta_i^-).
\end{align*}
Thus, we find a neural network in $\N_{1,\infty,1,c_o}^{2,\dd}$ representing $(g_+,g_-)$, where $d_i\le [r_1/2]+[r_2/2]+2m_1+2$.

When $k=K$, define $r_1^*=(K-1)[r_1/K], r_2^*=(K-1)[r_2/K], r^*=r_1+r_2$ and  \[g^*(\x)=(g^*_+(\x),g^*_-(\x))=\left(\frac{1}{\sum\limits_{i=1}^{r^*_1}c_i^+}\sum\limits_{i=1}^{r^*_1}c_i^+\si\left((\w_i^+)^T\x+b_i^+\right),\frac{1}{\sum\limits_{i=1}^{r^*_2}c_i^-}\sum\limits_{i=1}^{r^*_2}c_i^-\si\left((\w_i^-)^T\x+b_i^-\right)\right).\] By induction assumption, $g^*$ could be represented $h^* \in\N_{1,\infty,1,1}^{K-1,\dd^*}$, where $d_i^*\le [r_1^*/(K-1)]+[r_2^*/(K-1)]+2m_1+2$. In order to construct a WN-DNN representing $(g_+,g_-)$, we keep the first $K-1$ hidden layers of $h^*$ and build the $K$th hidden layer based on the output layer of $h^*$. 
Since the $(K-1)$th hidden layer contains both $\si\circ \x$ and $\si\circ(-\x)$. Thus except the original two neurons, we could add \[\{(\w_i^+)^T(\si\circ\x-\si\circ(-\x))+b^+_i: i=r_1^*+1,\cdots,r_1\},\]\[\{(\w_i^-)^T(\si\circ\x-\si\circ(-\x))+b_i^-: i=r_2^*+1,\cdots,r_2\},\si\circ\x),\si\circ(-\x)\] to the $K$th hidden layer. Note that $\norm{(b_i,2\w_i^T)}_1=1$, thus we does not increase the $L_{1,\infty}$ norm of the $K$th transformation by adding these neurons. 

We finally construct the output layer by 
\begin{align*}&\frac{\sum\limits_{i=1}^{r_1^*}c_i^+}{\sum\limits_{i=1}^{r_1}c_i^+}\si(g_+^*(\x))+\sum\limits_{i=r_1^*+1}^{r_1}\frac{c_i^+}{\sum\limits_{i=1}^{r_1}c_i^+}\si\left((\w_i^+)^T\x+b_i^+\right),\\
&\frac{\sum\limits_{i=1}^{r_2^*}c_i^-}{\sum\limits_{i=1}^{r_2}c_i^-}\si(g_-^*(\x))+\sum\limits_{i=r_2^*+1}^{r_2}\frac{c_i^-}{\sum\limits_{i=1}^{r_2}c_i^-}\si\left((\w_i^-)^T\x+b_i^-\right).
\end{align*}

Thus, we build a neural network in $\N_{1,\infty,1,1}^{K,\dd}$ representing $(g_+,g_-)$.  The width of the $i$th hidden layer $d_i\le [r_1/K]+[r_2/K]+2m_1+3$.
\end{proof}
\subsection{Proof for Theorem \ref{thm:approx}}
\begin{proof}
	Assume $f$ is an arbitrary function defined on $\R^{m_1}\to \R$, satisfying that $\norm{\x_1}_{\infty}\le 1$, $\norm{\x_2}_{\infty}\le 1$,  $f(\x_1) \le L$ and $\lvert f(\x_1)-f(\x_2)\rvert \le L\norm{\x_1-\x_2}_{\infty}$. Following \cite[Propositions 1 \& 6]{bach2017breaking}, for $c_o$ greater than a constant depending only on $m_1$,  a fixed $\gamma >0$, , there exists some function  $h(\x):\R^{m_1}\to\R=\sum\limits_{i=1}^rc_i\si (\w_i^T\x+b_i)$, satisfying that $\sum\limits_{i=1}^r |c_i|\le c_o$, $\norm{(b_i,\w_i^T)}_1=1$ and $r\le c_2(m_1)\gamma^{-\frac{2(m_1+1)}{m_1+4}}$, such that \[
	\sup\limits_{\norm{\x}_\infty\le 1}\vert f(\x)-h(\x)\rvert\le c_o\gamma+c_1(m_1)L(\frac{c_o}{L})^{-\frac{2}{m_1+1}}\log\frac{c_o}{L},\]
	where $c_1(m_1)$ and $c_2(m_1)$ are some constants depending only on $m_1$.  
	
	By taking $\gamma= c_1(m_1)(c_o/L)^{-1-2/(m_1+1)}\log\frac{c_o}{L}$, we have some function  $h(\x)=\sum\limits_{i=1}^rc_i\si (\w_i^T\x+b_i)$, satisfying that $\sum\limits_{i=1}^r |c_i|\le c_o$, $\norm{(b_i,2\w_i^T)}_1=1$ and \[r\le C_r(m_1)(\log \frac{c_o}{L})^{-2(m_1+1)/(m_1+4)}\left(\frac{c_o}{L}\right)^{2(m_1+3)/(m_1+4)},\]
	such that  \[
	\sup\limits_{\norm{\x}_\infty\le 1}\vert f(\x)-h(\x)\rvert\le C(m_1)L(\frac{c_o}{L})^{-\frac{2}{m_1+1}}\log\frac{c_o}{L},\]
	where $C_r(m_1)$ and $C(m_1)$ denote some constants that depend only on $m_1$.

	By Lemma \ref{lemma:approx:con:deep}, for any integer $k\in[1,r] $, this $h$ could be represented by a neural network in  $\N_{p,\infty,1,c_o}^{k,\dd^k}$, where 
	$\dd^k_0=m_1$, $\dd^k_i=[r/k]+2m_1+3$ for $i=1,\cdots,k$ and $\dd^k_{k+1}=1$.
\end{proof}

\end{document}